\documentclass{article}
\usepackage{graphicx}
\usepackage{amsmath}
\usepackage{amsfonts}
\usepackage{amssymb}
\usepackage{amsthm}
\usepackage{mathtools}
\usepackage{enumerate}
\usepackage{lscape}
\usepackage{longtable}
\usepackage{rotating}
\usepackage{multirow}
\usepackage{color}
\usepackage{url}
\usepackage{subfigure}
\usepackage{rotating}
\usepackage{comment}
\usepackage[utf8]{inputenc}
\usepackage{caption}
\usepackage{xcolor}

\newtheorem{assumption}{Assumption}

\newtheorem{definition}{Definition}
\newtheorem{proposition}{Proposition}
\newtheorem{lemma}{Lemma}
\newtheorem{theorem}{Theorem}
\newtheorem{remark}{Remark}

\newcommand{\CM}{\mathcal{M}}

\newcommand{\CK}{\mathcal{K}}

\newcommand{\CS}{\mathcal{S}}

\newcommand{\BR}{\mathbb{R}}
\newcommand{\BN}{\mathbb{N}}

\newcommand{\rd}{\mathrm{d}}
\newcommand{\bd}{\mathbf{d}}

\DeclareMathOperator{\grad}{grad}
\DeclarePairedDelimiter\floor{\lfloor}{\rfloor}

\newcommand{\Rexp}{\BR_{\text{exp.}}}
\newcommand{\Ran}{\BR_{\text{an.}}}
\newcommand{\Ranexp}{\BR_{\text{an.,exp.}}}
\newcommand{\Ralg}{\BR_{\text{semialg.}}}
\newcommand{\norm}[1]{\left\lVert#1\right\rVert}

\begin{document}

\title{Tame Riemannian Stochastic Approximation}
\author{Johannes Aspman, Vyacheslav Kungurtsev, Reza Roohi Seraji}

\maketitle

\begin{abstract}
We study the properties of stochastic approximation applied to a tame nondifferentiable function subject to constraints defined by a Riemannian manifold. The objective landscape of tame functions, arising in o-minimal topology extended to a geometric category when generalized to manifolds, exhibits some structure that enables theoretical guarantees of expected function decrease and asymptotic convergence for generic stochastic subgradient descent. Recent work has shown that this class of functions faithfully model the loss landscape of deep neural network training objectives, and the autograd operation used in deep learning packages implements a variant of subgradient descent with the correct properties for convergence. Riemannian optimization uses geometric properties of a constraint set to perform a minimization procedure while enforcing adherence to the the optimization variable lying on a Riemannian manifold. This paper presents the first study of tame optimization on Riemannian manifolds, highlighting the rich geometric structure of the problem and confirming the appropriateness of the canonical ``SGD'' for such a problem with the analysis and numerical reports of a simple Retracted SGD algorithm.
\end{abstract}


\section{Introduction}
Consider an optimization problem,
\begin{equation}\label{eq:prob1}
\min\limits_{x\in\mathcal{M}} F(x):=\mathbb{E}[f(x,\xi)]
\end{equation}
where $\mathcal{M}$ is a Riemannian manifold, and $F$ is continuous but not necessarily continuously differentiable, i.e., it is nonsmooth. We can assume that it is tame, however. In typical machine learning applications, $F$ encodes a loss function, and $x$ is a set of parameters in a model that are algorithmically driven to minimize the loss and thus best fit the data. The data distribution, empirical or otherwise, is defined by the stochastic sample $\xi$. Formal probabilistic details are given below.

The parameters are constrained to lie on a manifold $\mathcal{M}$, which encodes that the parameters are not confined to the usual free Euclidean space, but have some set of symmetries between the components that must be preserved. A popular visually intuitive example is in robotics: a robot often has the choice to rotate, possibly along multiple axes, suggesting the decision must lie in the symmetric rotation group. Other applications of Riemannian optimization for machine learning include
learning parameters of Gaussians mixtures, principal component analysis, and Wasserstein barycenters~\cite{hosseini2020recent}. A number of code packages have been developed to perform Riemannian optimization on machine learning tasks, for instance~\cite{miolane2020geomstats}, as well as the two extensions of the two most popular libraries for training neural networks, \texttt{PyTorch} (see \cite{kochurov2020geoopt}) and \texttt{TensorFlow} (see \cite{smirnov2021tensorflow}).

The canonical algorithm to consider is Retraction-SGD
\begin{equation}\label{eq:retsgd}
x_{k+1}=R_{x_k}(x_k+\alpha_k g_k),\quad g_k\sim \partial F(x_k,\cdot)
\end{equation}
where $g_k$ is meant to approximate an element of the subdifferential of $F$ at $x_k$ by sampling from the stochastic variable (typically a minibatch), i.e., it is a \emph{stochastic gradient}. The notion will be made more precise for the nonsmooth objectives in the sequel. The operator $R_{x}$ from base point $x$ is a \emph{retraction}, which can be understood as a projection onto the manifold $\mathcal{M}$. The step-size $\alpha_k$ can be varying per iteration, typically decreasing, or constant. 

The case of~\eqref{eq:prob1}, however with $F(x)$ (and $f$, with respect to $x$) being continuously differentiable is considered in~\cite{durmus2020convergence,durmus2021riemannian,shah2021stochastic}. See also ~\cite{iiduka2022riemannian} on stochastic fixed point iterations, which can apply to nonsmooth convex functions. In this context, we allow $\partial F$ to correspond to the Clarke subdifferential, the output of an automatic differentiation of a composite function or a differential path on a conservative vector field. These notions will be made precise, integrating the geometric and topological functionality of Riemannian and tame geometry.

For the Euclidean case, diminishing stepsize SGD has been studied by, e.g., Benaim and others~\cite{benaim2005stochastic,majewski2018analysis}. For a constant stepsize analysis for nonsmooth problems, see~\cite{bianchi2022convergence}.

For numerical examples of~\eqref{eq:prob1} with tame $F$, see~\cite{kungurtsev2022retraction}, which is a paper that considers the case wherein one do not have access to subgradient estimates but only function values, and so must employ a zero order method. 

The incorporation of the machinery for understanding tame functions to the case of general Riemannian manifolds presents an interesting and insightful geometric exercise. For a plurality of definitions and partial statements, there are sufficiently simple correspondences between the Euclidean and Riemannian formulations so as to facilitate more or less straightforward translation of existing results to the Riemannian case. However, there are a number of subtle but critical departure points for which care and technical nuance is necessary to adapt the machinery appropriately. 

We shall consider both the diminishing and constant step size case. The distinction is more than superficially algorithmic, rather the appropriate notion of convergence is distinct. In the first case, by directly applying the ``ODE'' approach of stochastic approximation, wherein the intended program is to show that a limiting sequence of the algorithm approaches a curve of some trajectory that satisfies a differential inclusion, whose equilibrium points are stationary with respect to~\eqref{eq:prob1}. In the second, one can observe that a constant stepsize implies that the random process is Markovian, permitting an analysis that treats the algorithm as a stochastic process whose long run behavior is intended to be characterized by ergodicity concentrated around a set of stationary points (along the lines of, for example~\cite{bianchi2022convergence} for the Euclidean case).

\subsection{Contributions}

\paragraph{Structure of the Theoretical Background} 
When working with non-smooth functions we of course need some generalised notion of derivatives. In this paper, we work with the notion of a conservative field, which was developed for the Euclidean setting in \cite{bolte2021conservative}, lifted to the Riemannian setting. The relevant results on this are presented in Sec. \ref{s:consriem}. In \cite{bolte2021conservative}, the main class of non-smooth functions considered was functions definable in some o-minimal structure on $\BR$. When we instead work on a Riemannian manifold, the analogous notion is that of analytic-geometric categories. As these are much less known in the optimization literature, for completeness, we present the relevant results in \ref{s:anageomcat}. We will then exclusively work in this setting and refer to functions definable in some analytic-geometric category as being \emph{tame}. 

Working in the Riemannian setting brings about a number of new challenges compared to the previous literature \cite{davis2020stochastic, bolte2021conservative, bianchi2022convergence}. The main reason for this is of course that $\CM$ is in general not a vector space, and the tangent spaces are different from the underlying manifold. This means that many of the definitions and proofs are not directly lifted and some extra subtleties arise compared to the Euclidean case, especially in relation to the flow of solutions to the differential inclusions considered in for example \cite{bianchi2022convergence}. These subtleties also arise in the analyses of the convergence properties of the stochastic approximation algorithm on a Riemannian manifold, both for the diminishing and constant stepsize settings.

\paragraph{Structure of the Diminishing Stepsize Convergence Theory} 

Our main result regarding the diminishing stepsize setting is Theorem \ref{th:convergence}. In order to prove the limiting behavior of the algorithm towards critical points of $f$, we first need to prove a series of results regarding asymptotic pseudo trajectories, following the outline of \cite{benaim2005stochastic}, see Theorems \ref{th:APT}-\ref{th:ict}. The main difference in these proofs compared to those of \cite{benaim2005stochastic} is the fact that we are generally not in a vector space anymore, and we need to use the properties of the parallel transport in order to compare different trajectories. Together with some additional Lemmas, we can then proceed by proving Theorem \ref{th:3.2} which then implies the main theorem \ref{th:convergence}. The arguments require a few corrections throughout to account for the curvature of the underlying manifold. 

\paragraph{Structure of the Constant Stepsize Convergence Theory} 
For the case of constant stepsize, our main result is Theorem \ref{th:convCSS}. In order to prove this, we first need to prove some properties of the SGD sequence on a manifold, Theorem \ref{th:tight}. After this we prove that the interpolated process defined by the SGD sequence converges in probability to the set of solutions to the underlying differential inclusion, see Theorem \ref{th:weakconv}. The vector space structure in the Euclidean case - wherein distances are simply the magnitude of positive/negative sums of vectors, is unavailable in the Riemannian case. A more constructive argument considering geodesic paths between carefully chosen virtual iterates must be developed.

\paragraph{Observations of the Numerical Results} 
We present the results of a series of computational experiments testing variants of Retracted SGD on several standard problems arising in Machine Learning. We demonstrate the successful asymptotic convergence of the method. In addition, we demonstrate the ergodicity by displaying the empirical distribution of randomized trial runs on the constraint manifold.  

\paragraph{Paper outline} The rest of the paper proceeds as follows. Section~\ref{s:back} reviews important notions of Riemannian geometry and measure-based probability theory on a Riemannian manifold. Section~\ref{s:consriem} presents the appropriate notions for well-definedness of tame functions on a manifold. Section~\ref{s:sa} summarizes our convergence results. We finish with some numerical illustrations. Technical proofs for some of the results appear in the appendix.

\section{Background on Riemannian Geometry and Probability Measures on a Manifold}\label{s:back}
We recall some background results on Riemannian geometry. Throughout, we let $\CM$ denote a smooth $m$-dimensional Riemannian manifold. The tangent space at a point $x\in\CM$ is denoted $T_x\CM$ and the tangent bundle of $\CM$ by $T\CM$. Similarly, the cotangent space at $x$ and the cotangent bundle are denoted $T^*_x\CM$ and $T^*\CM$, respectively. The metric on $\CM$, $g(\cdot,\cdot)$, or $g_x(\cdot,\cdot)$ when evaluated at a point $x\in\CM$, induces a norm $\norm{\cdot}_{g_x}\coloneqq \sqrt{g_x(\cdot,\cdot)}$. We will often drop the subscript $g_x$ in the notation when the metric is generic and the given point is obvious from the situation. For $v\in T_x\CM$ and $w\in T_x^*\CM$ we have the scalar product $\langle w, v\rangle$. The length of a piecewise smooth curve $\gamma\,:\,[a,b]\to\CM$ is defined as
\begin{equation}
    L(\gamma)=\int_a^b\norm{\dot\gamma(t)}_{g_{\gamma(t)}}dt.
\end{equation}
For two points $x,y\in\CM$, we denote the Riemannian distance from $x$ to $y$ by $\bd(x,y)$, 
\begin{equation}
    \bd(x,y)\coloneqq \inf \{L(\gamma)\,:\,\gamma\in\mathcal{A}_\infty,\,\gamma(a)=x,\,\gamma(b)=y\},
\end{equation}
where $\mathcal{A}_\infty$ denotes the set of all piecewise smooth curves. The Riemannian distance defines a metric space structure on $\CM$. 

We write $B(x_0,r)$ (resp. $\overline{B(x_0,r)}$) for the open (closed) ball of radius $r>0$ centered at $x_0\in\CM$, i.e. $B(x_0,r)=\{x\in\CM\,:\, \bd(x_0,x)<r\}$.



For a smooth curve $\gamma\,:\,I\to \CM $, we denote the parallel transport along $\gamma$ from $\gamma(a)$ to $\gamma(b)$, $a,b\in I$, as $P_{\gamma(a)\gamma(b)}^{\gamma}$. It is defined by
\begin{equation}
    P_{\gamma(a)\gamma(b)}^{\gamma}(v)\coloneqq V(\gamma(b)),\qquad \text{for every }v\in T_{\gamma(a)}\CM,
\end{equation}
where $V$ is the unique parallel vector field along $\gamma$ with $V(\gamma(a))=v$. When $\gamma$ is the unique minimizing geodesic between $x$ and $y$ we simply write $P_{xy}$.

The exponential map $\exp_x\,:\,T_x\CM\to\CM$ projects a vector from the tangent space to the manifold along a geodesic. 

Throughout, we denote by $\mathcal{B}(\CM)$ the Borel $\sigma$-algebra on $\CM$. Let $\mathcal{L}(\CM)$ be the Lebesgue $\sigma$-algebra on $\CM$. A subset $A\subset \CM$ is in $\mathcal{L}(\CM)$ if, for any chart $(U,\varphi)$, $\varphi(A\cap U)$ is a Lebesgue-measurable subset of $\BR^m$. Note that $\mathcal{L}(\CM)\supseteq \mathcal{B}(\CM)$. For any set $A\subset U$, with $A\in\mathcal{L}(\CM)$, we have a unique measure defined by
\begin{equation}
    \lambda(A)=\int_{\varphi(A)}\sqrt{g}\rd \lambda_L,
\end{equation}
where $g=\det g_{ij}$ is the determinant of the metric in local coordinates and $\lambda_L$ is the Lebesgue measure on $\BR^m$. Since this induces a volume element for each tangent space, we also get a measure on the whole manifold $\CM$, which we denote $\lambda\coloneqq \lambda(\CM)$. We can then define a probability space $(\Omega, \mathcal{B},\mu)$ on $\CM$. 

The set of all probability measures on $\mathcal{B}(\CM)$ is denoted $\mathcal{P}(\CM)$, and for a subset $A\subset \CM$ we write
\begin{equation}
    \mathcal{P}_{\text{abs.}}(A)=\{\nu\in\mathcal{P}(\CM)\,:\,\nu\ll \lambda(A) \text{ and } \text{supp}(\nu)\subset A\},
\end{equation}
where, as usual, $\nu\ll\lambda$ denotes that $\nu$ is absolutely continuous with respect to $\lambda$. Finally, we write $\mathcal{P}_1(\CM)\coloneqq \{\nu\in\mathcal{P}(\CM)\,:\,\int \bd(y,z)^2\nu(\rd z)<\infty\}$, where $y\in\CM$ is a fixed point.


A random primitive on $\CM$ is a Borelian function $X$ from $\Omega$ to $\CM$, with probability density function, $p_X$ defined by
\begin{equation}
    \begin{aligned}
        &\mu(X\in \mathcal{X})=\int_{\mathcal{X}}p_X(y)\rd\lambda(y),\\
        &\mu(\CM)=\int_{\CM}p_X(y)\rd\lambda(y)=1,
    \end{aligned}
\end{equation}
for all $\mathcal{X}$ in the Borelian tribe of $\CM$. There is some subtlety regarding the choice of metric to use when defining the pdf on a manifold, for a discussion on this we refer to \cite{pennec2004probabilities}. For a Borelian real valued function $\phi(x)$ on $\CM$ we calculate the expectation value by
\begin{equation}
    \mathbb{E}[\phi(X)]=\int_{\CM}\phi(y)p_X(y)\rd\lambda(y).
\end{equation}
We further define the variance of a random primitive $X$ as
\begin{equation}
    \sigma_X^2(y)=\int_{\CM} \bd(x,y)^2p_X(z)\rd\lambda(z),
\end{equation}
where $y$ is now a fixed primitive.

We will furthermore make the following assumption on the geometry throughout. 
\begin{assumption}[Geodesic completeness]\label{assum_geodesicComplete}
$\CM$ is a connected geodesically complete Riemannian manifold. This makes the exponential map well-defined over the tangent bundle $T\CM$. 
\end{assumption}

We will study retractions on $\CM$, which we define as follows.
\begin{definition}[Retraction, Def. 2 in \cite{shah2021stochastic}]
    A retraction on $\CM$ is a smooth mapping $\mathcal{R}\,:\,T\CM\to\CM$ such that, for $x\in\CM$,
    \begin{enumerate}
        \item $\mathcal{R}_x(0_x)=x$, where $\mathcal{R}_x$ is the restriction of the retraction to $T_x\CM$ and $0_x$ denotes the zero element of $T_x\CM$;
        \item with the canonical identification $T_{0_x}T_x\CM\cong T_x\CM$, $\mathcal{R}_x$ satisfies
        \begin{equation}
            D\mathcal{R}_x(0_x)=\text{Id}_{T_x\CM},
        \end{equation}
        where $\text{Id}_{T_x\CM}$ denotes the identity operator on $T_x\CM$. 
    \end{enumerate}
\end{definition}
With this definition in mind we are interested in studying the process defined by
\begin{equation}\label{eq:retraction_process}
    x_{k+1}=\mathcal{R}_{x_k}(-\alpha_k(g_k(x_k,\xi_{k+1}))),
\end{equation}
for some $g_k\in T_{x_k}\CM$ that will be specified later on. We will make the assumption throughout that $g_k$ has zero mean and finite variance. A classical example of a retraction is the exponential map.

\section{Conservative set valued fields on a Riemannian manifold}\label{s:consriem}
Bolte and Pauwels, \cite{bolte2021conservative}, introduced the important concept of a conservative set-valued field. For completeness we list the relevant definitions and properties of these fields, lifted to the Riemannian setting.

\subsection{Absolutely continuous curves and conservative fields}


An important notion in the following will be that of an absolutely continuous curve. We follow \cite{Burtscher2012LengthSO} and start by defining an absolutely continuous function on Euclidean space. To this end, let $I\subset \BR$ be a closed interval. We call a function $f: I\to\BR$ absolutely continuous (on $I$) if for all $\varepsilon>0$ there exists a $\delta>0$ such that for any $n\in\BN$ and any selection of disjoint intervals $\{(a_i,b_i)\}_{i=1}^{n}$ with $[a_i,b_i]\subseteq I$, whose overall length is $\sum_{i=1}^n|b_i-a_i|<\delta$, $f$ satisfies 
\begin{equation}
    \sum_{i=1}^n|f(b_i)-f(a_i)|<\varepsilon.
\end{equation}
We furthermore call a function $f\,:\, I\to \BR^n$ \emph{locally absolutely continuous} if it is absolutely continuous on all closed subintervals $[a,b]\subseteq I$. 

Now consider the case of a Riemannian manifold $\CM$. We call a continuous map $\gamma: I\to\CM$ an absolutely continuous curve if, for all $\varepsilon>0$ there exists a $\delta>0$ such that for any $n\in\BN$ and any selection of disjoint intervals $\{(a_i,b_i)\}_{i=1}^{n}$ with $[a_i,b_i]\subseteq I$, whose overall length is $\sum_{i=1}^n|b_i-a_i|<\delta$, $\gamma$ satisfies
\[
\sum_{i=1}^n \bd(\gamma(b_i),\gamma(a_i))<\epsilon
\]
We can see that there is a natural correspondence to the Euclidean setting when applying local diffeomorphisms. That is, for any chart $(U,\varphi)$ of $\CM$, the composition
\[
\varphi\circ \gamma :\gamma^{-1}(\gamma(I)\cap U) \to \varphi(U)\subseteq \BR^m,   
\]
is locally absolutely continuous. Absolutely continuous curves admit a derivative, $\dot\gamma(t)\in T_{\gamma(t)}\CM$, a.e., and the (arc)length, $l(\gamma)$, is well-defined:
\[
l(\gamma):= \int_I \|\dot{\gamma}(t)\|\rd t
\]

We are now ready to lift the relevant notions from \cite{bolte2021conservative} to the Riemannian setting. The good news is that everything generalizes more or less straightforwardly, as was already pointed out in a footnote of \cite{bolte2021conservative}. 

First of all, we have the following lemma, whose proof goes through without any modifications.

\begin{lemma}[Lemma 1 of \cite{bolte2021conservative}]\label{lemma1BP}
    Let $D: \CM\rightrightarrows T^*\CM$ be a set-valued map with nonempty compact values and closed graph. Let $\gamma :[0,1]\to\CM$ be an absolutely continuous curve. Then
    \begin{equation}
        t\mapsto \max_{v\in D(\gamma(t))} \langle v,\dot\gamma(t)\rangle,
    \end{equation}
    defined almost everywhere on $[0,1]$, is measurable. 
\end{lemma}

Two central objects we will be concerned with are the conservative set-valued maps and their potential functions. 
\begin{definition}[Conservative set-valued field, cf. Def. 1 of \cite{bolte2021conservative}]
    Let $D: \CM\rightrightarrows T^*\CM$ be a set-valued map. We call $D$ a \emph{conservative field} whenever it has a closed graph, nonempty compact values and for any absolutely continuous loop $\gamma: [0,1]\to \CM$, we have 
    \begin{equation}
        \int_0^1 \max_{v\in D(\gamma(t))}\langle \dot\gamma(t),v\rangle \rd t=0.
    \end{equation}
    Equivalently, we could use the minimum in the definition.
\end{definition}

\begin{definition}[Potential functions of conservative fields, cf. Def. 2 of \cite{bolte2021conservative}]
    Let $D: \CM\rightrightarrows T^*\CM$ be a conservative field. A function $f: \CM\to\BR$ defined through any of the equivalent forms
    \begin{equation}
        \begin{aligned}
            f(y)=& f(x)+\int_0^1\max_{v\in D(\gamma(t))}\langle \dot\gamma(t),v\rangle \rd t \\
           =& f(x)+\int_0^1\min_{v\in D(\gamma(t))}\langle \dot\gamma(t),v\rangle \rd t \\
           =& f(x)+\int_0^1\langle \dot\gamma(t),D(\gamma(t))\rangle\rd t,
        \end{aligned}
    \end{equation}
    for any absolutely continuous $\gamma$ with $\gamma(0)=x$ and $\gamma(1)=y$ is called a \emph{potential function for $D$}. It is well-defined and unique up to a constant. We will sometimes also say that $D$ is a conservative field for $f$. 
\end{definition}

We further say that $f$ is \emph{path differentiable} if $f$ is the potential of some conservative field $D$ and that $x$ is a $D$-critical point for $f$ if there exists $v \in D(x)$ with $\langle v,w\rangle=0$ for all $w\in T\mathcal{M}$

Recall that a function $f$ is Lipschitz continuous with constant $K$ on a given subset $S$ of $\CM$ if $|f(x)-f(y)|\leq K\bd(x,y)$, for every $x,y\in S$. We say that $f$ is Lipschitz at $x\in\CM$ if for all $y\in S(x)$, an open neighborhood of $x$, $f$ satisfies the Lipschitz condition for some $K$. Finally $f$ is called locally Lipschitz on $\CM$ if it is Lipschitz continuous at all $x\in\CM$. Note that a potential function of a conservative field, as described above, is locally Lipschitz.

One of the important properties of the conservative fields is that they come equipped with a chain rule. 
\begin{lemma}[Chain rule, cf. Lemma 2 of \cite{bolte2021conservative}]\label{th:chain}
    Let $D: \CM\rightrightarrows T^*\CM$ be a locally bounded, graph closed set-valued map and $f: \CM\to\BR$ a locally Lipschitz continuous function. Then $D$ is a conservative field for $f$ if and only if, for any absolutely continuous curve $\gamma: [0,1]\to \CM$, the function $t\mapsto f(\gamma(t))$ satisfies
    \begin{equation}\label{eq:chainrule}
        \frac{\rd}{\rd t}f(\gamma(t))=\langle v,\dot\gamma(t)\rangle,\qquad \forall v\in D_f(\gamma(t)),
    \end{equation}
    for almost all $t\in [0,1]$. 
\end{lemma}

Finally, we will need the following theorem, which we prove in Appendix \ref{app:proof_prel}.
\begin{theorem}[Cf. Theorem 1 in \cite{bolte2021conservative}]\label{th:gradEverywhere}
Consider a conservative field $D: \CM\rightrightarrows T^*\CM$ for the potential $f:\CM\to\BR$. Then $D=\{\rd f\}$ almost everywhere. 
\end{theorem}

A central object in non-smooth analysis is the Clarke subdifferential. We define it for a Riemannian manifold following \cite{Hosseini2011Generalized}.  To this end, we first define the generalized directional derivative.

\begin{definition}[Generalized directional derivative and Clarke subdifferential]
    Let $f: \CM\to\BR$ be a locally Lipschitz function and $(U,\varphi)$ a chart at $x\in\CM$. The \emph{generalized directional derivative} of $f$ at $x$ in the direction $v\in T_x\CM$, denoted $f^\circ(x;v)$, is then defined by
    \begin{equation}
        f^\circ(x;v)\coloneqq \limsup_{y\to x, t\searrow 0}\tfrac{f\circ \varphi^{-1}(\varphi(y)+t\rd \varphi(x)(v))-f\circ\varphi^{-1}(\varphi(y))}{t}.
    \end{equation}

    The Clarke subdifferential of $f$ at $x$, denoted $\partial f(x)$, is furthermore the subset of $T^*_x\CM$ whose support function is $f^\circ(x;\cdot)$. 
\end{definition}

As a corollary of Theorem \ref{th:gradEverywhere}, the Clarke subdifferential gives a minimal convex conservative field:

\begin{theorem}[Cf. Corollary 1 in \cite{bolte2021conservative}]\label{corr:Clarkeae}
    Let $f:\CM\to \BR$ allowing a conservative field $D:\CM\rightrightarrows T^*\CM$. Then $\partial f$ is a conservative field for $f$, and for all $x\in\CM$
    \begin{equation}
        \partial f(x)\subset \text{conv}\,(D(x)).
    \end{equation}
\end{theorem}

\begin{proof}
    The idea is the same as in the proof of Corollary 1 in \cite{bolte2021conservative}. Fix an $S\subset\CM$, a full measure set such that $ D= \rd f$ on $S$. From Lemma 5.5 of \cite{Hosseini2011Generalized} we have
    \begin{equation}
        \partial f(x)=\overline{\text{co}}\{\lim_{k\to\infty}\rd f(x_k)\,:\,\{x_k\}\subset S,\,x_k\to x\}.
    \end{equation}
    But $D$ has a closed graph and on $S$ we have $D= \rd f$, so the result follows directly. 
\end{proof}

\begin{remark}

An important application of conservative fields is to non-smooth automatic differentiation, as discussed in \cite{bolte2021conservative}. Here, the more standard generalized subdifferentials are not enough to perform the analysis. In fact, we have the following result, relating backpropagation as implemented in standard deep learning automatic differentiation packages, to conservative fields.

\begin{proposition}
  [Cf. Theorem 8 and Corollary 6 of \cite{bolte2021conservative}] Let $f$ be tame and computable as a directed graph of simple function compositions.  Backpropagation on $f$ is a conservative field.
\end{proposition}
   
\end{remark}

\subsection{Analytic-geometric categories and stratifications}\label{s:anageomcat}
The power of o-minimal structures and tame geometry in optimization is by now well-known, \cite{Bolte2009Semismooth,davis2020stochastic,bolte2021conservative,difonzo2022stochastic,bareilles2023piecewise}. However, o-minimal structures are defined on Euclidean spaces $\BR^n$. To generalize to the manifold setting \cite{Dries1996GeometricCA} put forward the definition of an analytic-geometric category. Roughly, we can say that analytic-geometric categories are locally given by o-minimal structures extending $\Ran$. Due to this property, the objects of the analytic-geometric categories share most of the important and useful properties of o-minimal structures. See Appendix \ref{sec:ominimal} for a brief discussion on o-minimal structures.

\begin{definition}[Analytic-geometric category, \cite{Dries1996GeometricCA}]
An analytic-geometric category, $\mathcal{C}$, is given if each manifold $\CM$ is equipped with a collection $\mathcal{C}(\CM)$ of subsets of $\CM$ such that the following conditions hold for each manifolds $\CM$ and $\mathcal{N}$:
\begin{enumerate}\setlength\itemsep{0.5em}
	        \item[{1)}] $\mathcal{C}(\CM)$ is a boolean algebra of subsets of $\CM$, with $\CM\in\mathcal{C}(\CM)$;
	        \item[{2)}] if $A\in\mathcal{C}(\CM)$, then $A\times \BR\in \mathcal{C}(\CM\times \BR)$; 
	        \item[{3)}] if $f:\CM\to\mathcal{N}$ is a proper analytic map and $A\in\mathcal{C}(\CM)$, then $f(A)\in\mathcal{N}$;
	        \item[{4)}] if $A\subseteq\CM$ and $\{U_i\}_{i\in I}$ is an open covering of $\CM$, then $A\in\mathcal{C}(\CM)$ iff $A\cap U_i\in\mathcal{C}(U_i)$ for all $i\in I$;
                \item[{5)}] every bounded set in $\mathcal{C}(\BR)$ has finite boundary.
	    \end{enumerate}
\end{definition}
Objects of this category are pairs $(A,\CM)$, where $\CM$ is a manifold and $A\subset \mathcal{C}(\CM)$. We refer to an object $(A,\CM)$ as the $\mathcal{C}$-set $A$ (in $\CM$). The morphisms $(A,\CM)\to (B,\mathcal{N})$ are continuous mappings $f:A\to B$ and referred to as $\mathcal{C}$-maps. Their graphs belong to $\mathcal{C}(\CM\times \mathcal{N})$. We will borrow the terminology from o-minimal structures and say that a set (function) is definable in an analytic-geometric category $\mathcal{C}(\CM)$ if it (its graph) belongs to $\mathcal{C}(\CM)$. More generally, sets and functions definable in some analytic-geometric category are simply referred to as being \emph{tame}, when no specific such category is referenced.  

An important property of sets definable both in o-minimal structures and analytic-geometric categories is that they are Whitney stratifiable. 
\begin{definition}[Whitney stratification]
	A Whitney $C^k$ stratification $M=\{M_i\}_{i\in I}$ of a set $A$ is a partition of $A$ into finitely many non-empty $C^k$ submanifolds, or strata, satisfying:
	\begin{itemize}
	    \item \textbf{Frontier condition:} For any two strata $M_i$ and $M_j$, the following implication holds,
	    \begin{equation}
	        \overline{M}_i\cap M_j\neq \emptyset
	        \implies M_j\subset \overline{M}_i.
	    \end{equation}
	    
	    \item \textbf{Whitney condition (a):} For any sequence of points $x_k$ in a stratum $M_i$ converging to a point $x$ in a stratum $M_j$, if the corresponding normal vectors $v_k\in N_{M_i}(x_k)$ converge to a vector $v$, then the inclusion $v\in N_{M_j}(x)$ holds.
	\end{itemize}
	\end{definition}
A $C^k$ Whitney stratification of a function $f: \mathcal{S}\to\mathcal{N}$, for $\mathcal{S}\subset \CM$ closed, is a pair $(S,N)$ of Whitney stratifications of $\mathcal{S}$ and $\mathcal{N}$, respectively, such that for each $P\in\mathcal{S}$ the map $f|_P: P\to \mathcal{N}$ is $C^k$ with $f(P)\in N$ and $(\text{rk}\, f|_P)(x)=\text{dim}\, f(P)$ for all $x\in P$. Here, $\text{rk}\,f(x)\coloneqq T_xf$, where $T_xf:T_x\mathcal{S}\to T_{f(x)}\mathcal{N}$ is the induced linear map between tangent spaces \cite{Dries1996GeometricCA}. Theorem D.16 of \cite{Dries1996GeometricCA} proves that sets definable in analytic-geometric categories are Whitney stratifiable. 



Furthermore, we have the following results on Whitney stratifiable functions. 

\begin{definition}[Variational stratification]
    Let $f: \CM\to\BR$ be locally Lipschitz continuous, $D:\CM\rightrightarrows T^*\CM$ a set-valued map and let $k\geq 1$. We say that $(f,D)$ has a $C^k$ \emph{variational stratification} if there exists a $C^k$ Whitney stratification $M$ of $\CM$ such that $f$ is $C^k$ on each stratum and for all $x\in\CM$:
    \begin{equation}
        \text{Proj}_{T_xM_x}D(x)=\{\rd_xf(x)\},
    \end{equation}
    where $\rd_xf(x)$ is the differential of $f$ restricted to the active strata $M_x$ containing $x$. 
\end{definition}

 \begin{proposition}[Variational stratification for definable conservative fields]\label{th:VarStrat}
     Let $D:\CM\rightrightarrows T^*\CM$ be a definable conservative field having a definable potential $f:\CM\to\BR$. Then $(f,D)$ has a $C^k$ variational stratification. 
\end{proposition}

The Whitney stratifiability of the $\mathcal{C}$-maps allows us to make some important claims. Recall that we call  The following will be important:

\begin{proposition}[Non-smooth Morse-Sard, cf. Theorem 5 in \cite{bolte2021conservative}]\label{th:sard}
    Let $D:\CM\rightrightarrows T^*\CM$ be a conservative field for $f:\CM\to\BR$ and assume that $f$ and $D$ are definable. Then the set of $D$-critical values, $\{f(x):x\in\CM\, s.t.\,\exists v\in D(x) \,s.t\, \langle v,w\rangle=0,\,\forall w\in T\CM\}$, is finite. 
\end{proposition}

\section{Stochastic Approximation on a Manifold}\label{s:sa}
We are now ready to discuss stochastic approximation on a Riemannian manifold. We will start by introducing a few important definitions that will play important roles in the coming analyses together with some assumptions, then discuss the stochastic approximation with diminishing step-size, before finally turning to the case with constant step-size.

\subsection{Introductory definitions and general assumptions}

Consider the metric space given by the set of continuous functions $C(\mathbb{R},\mathcal{M},\bd_C)$ endowed with the metric of uniform convergence on compact sets,
\[
\bd_C(x(t),y(t)) \coloneqq \sum\limits_{k=1}^\infty \frac{1}{2^k}\min\left(\int_{-k}^k \bd(x(t),y(t))dt ,1\right)
\]

Given a set-valued map $G:\mathcal{M}\rightrightarrows T\mathcal{M}$, we call an absolutely continuous curve $\gamma \,:\,[0,a]\to \CM$ a solution to the differential inclusion
\begin{equation}\label{eq:diffInclusion}
    \dot\gamma(t)\in G(\gamma(t)),\qquad x_0\in \CM,
\end{equation}
with initial condition $x_0$, if $\gamma(0)=x_0$ and the inclusion holds for almost all $t\in[0,a]$ \cite{Pouryayevali2019}. If $G$ is an upper semicontinuous set-valued function with compact and convex values, then for any $v\in G(x_0)$ the differential inclusion has a local solution with $\gamma(0)=x_0$ and $\dot\gamma(0)=v$, Theorem 6.2 in \cite{Ledyaev2007Nonsmooth}. 

Observe that we can take $G^*=-\text{conv}(D(\gamma(t)))$, which is a compact and convex subset of a linear vector space. Recalling that there exists an isomorphism $\iota:T^*\CM\to T\CM$ (typically called the \emph{musical isomorphism}, see, e.g., pp. 341–343 in ~\cite{lee2012smooth}), we define
$G(\gamma(t))=-\iota\left(\text{conv}(D(\gamma(t)))\right)$, which satisfies the above conditions, and so $\dot\gamma\in G(\gamma(t))$ is a differential inclusion with the aforementioned guarantees for the solution to exist.

Let us furthermore denote by $\mathcal{S}_{G}(A)$ the set of solutions to \eqref{eq:diffInclusion} with initial points in $A\subset \CM$, and $\widehat{\mathcal{S}}_G(A)\subset \mathcal{S}_G(A)$ the subset of solutions that stays in $A$. Finally, we denote  by $\mathcal{S}=\bigcup_x \mathcal{S}_x$ the set of all solutions to \eqref{eq:diffInclusion}.

We furthermore define, following Def. 2 of \cite{benaim2005stochastic}:
\begin{definition}[Perturbed solution, \cite{benaim2005stochastic}]
    A continuous function $\tilde\gamma\,:\,[0,a]\to\CM$ is called a perturbed solution of the differential inclusion \eqref{eq:diffInclusion} if it satisfies the following:
    \begin{enumerate}
        \item $\tilde\gamma$ is absolutely continuous;
        \item there exists a family of curves $u_t\,:\, \left\{\BR\to \CM\right\}^{\BR}$ defined for $s\in[0,T_t]$
        \begin{enumerate}
            \item locally integrable, i.e., the family has integrable arclengths, or \[\int_{s_1}^{s_2}\int_{0}^{T_t}\norm{\dot{u}_t(s)}_{g(u_t(s))} \rd s\rd t\] is finite for any finite $s_1,s_2>0$, and,
            \item For all $T>0$, it holds that, 
        \[
        \lim\limits_{s_1\to\infty}\sup\limits_{0\le v\le T}\int_{s_1}^{s_1+v}\int_{0}^{T_t}\norm{\dot{u}_t(s)}_{g(u_t(s))}\rd s\rd t = 0
        \]
        \item In addition, $\{u_t\}$ satisfies
        \begin{equation}
            P^{u_t}_{\tilde\gamma(t) u_t(s)}\dot{\tilde\gamma}(t)-\dot{u}_t(s)\in G^{\delta(t)}(\tilde\gamma(t)),\quad \forall s\in[0,T_t]
        \end{equation}
        for almost all $t>0$, with $\delta(t):[0,\infty)\to\mathbb{R}$ satisfying $\delta(t)\to 0$ and $G^{\delta}$ defined to be: 
        \begin{equation}
            \begin{aligned}
                G^{\delta}(x)\coloneqq \{ &y\in T_x\CM\,:\,\exists z\in\CM\,:\, \bd(z,x)<\delta,\\
                &\inf_{a\in G(z)}\norm{P_{xz}y-a}<\delta\}.
            \end{aligned}
        \end{equation}
        
        \end{enumerate}
    \end{enumerate}
\end{definition}

The limit set of a perturbed solution $\tilde\gamma$ is given by
\begin{equation}
    L(\tilde\gamma) = \bigcap_{t\geq 0}\overline{\{\tilde\gamma(s),\, s\geq t \}}.
\end{equation}

We introduce the notation $\Phi_t(x)=\{ \gamma(t):\,\gamma$ is a solution to \eqref{eq:diffInclusion} with $\gamma(0)=x\}$ and $\Phi=\{\Phi_t\}_{t\in\BR}$. This dynamical system has the following properties, that are straightforward to derive from the definition:
\begin{enumerate}
    \item $\Phi_0(x)=\{x\}$;
    \item $\Phi_t(\Phi_s(x))=\Phi_{t+s}(x)$, for all $t,s\geq0$;
    \item $y\in\Phi_t(x)\implies x\in\Phi_{t}(y)$ for any $x,y\in\CM$ and $t\in\BR$;
    \item $(x,t)\mapsto\Phi_t(x)$ is a closed set-valued map with compact values. 
\end{enumerate}

Let us further introduce the following, see for example \cite{boothby2003introduction},
\begin{definition}[Local flow on a manifold]
    A local flow on $\CM$ is a smooth map $\phi: U\to\CM$, where $U\subset \BR\times\CM$ is an open neighborhood of $\{0\}\times \CM$ such that
    \begin{enumerate}
        \item $\phi(0,x)=x$; and
        \item $\phi(s,\phi(t,x))=\phi(s+t,x)$, whenever both sides exists.
    \end{enumerate}
\end{definition}
It is thus obvious from this definition that $\Phi_t(x)$ is a local flow on $\CM$.

\begin{definition}[Invariant sets, Def. 8 in\cite{shah2021stochastic}]
     A set $A$ is said to be invariant under the flow $\Phi_x(\cdot)$ if $x\in A$ implies that $\Phi_t(x)\in A$ $\forall t\in\BR$. 
\end{definition}


\begin{definition}[Definition 9 of \cite{shah2021stochastic}]
    Let $\Phi_t(x)$ be a flow on a metric space $(\CM,\bd(\cdot,\cdot))$. Given $\varepsilon>0$, $T>0$ and $x,y\in\CM$ an $(\varepsilon,T)$ chain from $x$ to $y$ with respect to $\Phi_t(x)$ is a pair of finite sequences $x_1=x$, $x_2,\dots, x_{n-1},\, x_n=y$ and $t_1,\dots,t_{n-1}\in[T,\infty)$ such that 
    \begin{equation}
        \bd(\Phi_{t_i}(x_i),x_{i+1})<\varepsilon,\qquad \forall i=1,\dots,n-1.
    \end{equation}
    A set $A$ is called internally chain transitive for the flow if for any choice of $x$, $y$ in this set and any $\varepsilon$, $T$ as above, there exists an $(\varepsilon,T)$ chain for $A$. 
\end{definition}

By Lemma 3.5 in \cite{benaim2005stochastic}, internally chain transitive sets are invariant (the proof does not change, modulo the incorporation of the metric as opposed to a Euclidean norm).

Now, let $(\Omega,\mathcal{J}, \mu)$ be a probability space on $\CM$, with $\mathcal{J}$ being $\mu$-complete. Furthermore, let $(\tilde\Omega,\mathcal{F}, \mathbb{P}^\nu)$ be a probability space, where $\nu$ is a probability measure on $\mathcal{B}(\CM)$, $\tilde\Omega=\CM\times \Omega^{\BN}$, $\mathcal{F}=\mathcal{B}(\CM)\otimes \mathcal{J}^{\otimes \BN}$ and $\mathbb{P}^\nu=\nu\otimes\mu^{\otimes\BN}$. The canonical process on $\tilde\Omega\to\CM$ is denoted $(x_0,(\xi_n)_{n\in\BN^*})$.

Let us define an SGD sequence in the following way
\begin{definition}[SGD sequence, Def. 2 in \cite{bianchi2022convergence}]
    Let $f$ be $\mathcal{B}(\CM)\otimes\mathcal{J}/\mathcal{B}(\CM)$-measurable, and assume $f(\cdot,s)$ satisfies \ref{assum_tameness} for any $s\in\Omega$. A sequence $\{x_n\}_{n\in\BN^*}$ of functions on $\tilde \Omega\to\CM$ is called an SGD sequence for $f$ with steps $\alpha_n>0$ if there exists a selection, $\phi$, of a conservative field $D$ for $f$ such that 
    \begin{equation}\label{eq:sgdseq}
        x_{n+1}=\exp_{x_n}\left[-\alpha_n \phi(x_n,\xi_{n+1})\right],\quad \forall n\geq 0.
    \end{equation}
\end{definition}

We then make the following assumptions.
\begin{assumption}\label{assum_f}
    The function $f:\CM\times\Omega\to\BR$ having an SGD sequence satisfies the following. 
    
    \begin{enumerate}
        \item There exists a measurable function $\kappa\,:\,\CM\times\Omega\to\BR_+$ such that for each $x\in\CM$ we have $\int \kappa(x,s)\mu(\rd s)<\infty$ and there exists an $\varepsilon>0$ for which
        \begin{equation}
            \forall y,z\in B(x,\varepsilon),\,\forall s\in\Omega, \,|f(y,s)-f(z,s)|\leq \kappa(x,s)\bd(y,z).
        \end{equation}
        i.e., $f(\cdot,s)$ is geodesically $\kappa(\cdot,s)$-Lipschitz for all $s\in\Omega$.

        \item For all $x\in\CM$, $f(x,\cdot)$ is $\mu$-integrable. 

        \item There exists a point $0_{\CM}\in\mathcal{M}$ and a constant $C\geq 0$ such that $\int \kappa(x,s)\mu(\rd s)\le C(1+\bd(0_{\CM},x))$ for all $x\in\CM$. 

        \item For each compact set $\mathcal{K}\subset \CM$, $\sup_{x\in\mathcal{K}}\int \kappa(x,s)^2\mu(\rd s)<\infty$. 

    \end{enumerate}
\end{assumption}

The first two assumptions assures us that the function
\begin{equation}
    F\,:\,\CM\to\BR,\qquad x\mapsto \int f(x,s)\mu(\rd s),
\end{equation}
is locally Lipschitz on $\CM$.

Furthermore, we will make the following assumptions on the regularity of $f$:
\begin{assumption}[Tameness]\label{assum_tameness}
    We assume the following.
    \begin{enumerate}
        \item The function $f:\CM\times\Omega\to\BR$ is locally Lipschitz and tame, or more precisely $f(\cdot,s)$ is tame for each $s\in\Omega$. 
        \item The function $F:\CM\to\BR$, $x\mapsto \int f(x,s)\mu(\rd s)$ is tame. 
        \item The conservative fields $D_s:\CM\rightrightarrows T^*\CM$ belonging to the potential $f(\cdot,s)$ (for each $s\in\Omega$) are tame. 
    \end{enumerate}
\end{assumption}

Through the Rademacher theorem, see for example \cite[Thm. 5.7]{azagra2005nonsmooth} for the Riemannian case, the assumption of local Lipschitzness tells us that $f(\cdot,s)$ is almost everywhere differentiable (for each $s$).

\subsection{Stochastic Approximation on a Manifold - Diminishing Stepsize Algorithm}
Now we consider the Retracted SGD algorithm~\eqref{eq:retsgd}
\begin{equation}\label{eq:retsgddimsz}
x_{k+1}=R_{x_k}(x_k+\alpha_k g_k),\quad g_k\sim \partial F(x_k,\cdot),\quad\alpha_k\to 0^+
\end{equation}

Formally, let us make the following assumptions on the retraction process \eqref{eq:retraction_process}:
\begin{assumption}\label{assumptionsincl}
    \begin{enumerate}
        \item The steps $\{\alpha_k\}_{n\in\BN^*}$ form a sequence of non-negative numbers such that
        \begin{equation}
            \lim_{k\to\infty}\alpha_k=0,\quad \sum_k\alpha_k=\infty\quad \sum_k\alpha_k^2<\infty.
        \end{equation}
    
        \item For all $T>0$ and any $x\in\CM$
        \begin{equation}
            \begin{aligned}
                \limsup_{n\to\infty}\Big\{ \sum_{i=n}^{k-1}\alpha_{i+1}g(\iota G(x_{i+1}),\iota g_{i+1}):k=n+1,\dots, m(\tau_n+T) \Big\}=0,
            \end{aligned}
        \end{equation}
        with 
        \begin{equation}
            m(t)=\sup\{k\geq 0\,:\,t\geq \tau_k\},\quad \tau_n=\sum_{i=1}^n\alpha_i,
        \end{equation}
        and $\tau_0=0$. 
        \item $\sup_n \bd(x_n,z)<\infty$ for any point $z\in\CM$. 
    \end{enumerate}
\end{assumption}

The equation
\begin{equation}
    \Theta^t(\gamma)(s)=\gamma(s+t)
\end{equation}
defines a translation flow $\Theta^t\,:\, C(\BR,\CM)\times\BR\to C(\BR,\CM)$. We call a continuous curve $\zeta:\BR_+\to\CM$ an \emph{asymptotic pseudo trajectory} (APT) for $\Phi$ if
\begin{equation}\label{eq:APTdef}
    \lim_{t\to\infty}\bd_C(\Theta^t(\zeta),\mathcal{S}_{\zeta(t)})=0,
\end{equation}
where, as before, $\mathcal{S}_G$ denotes the set of solutions to \eqref{eq:diffInclusion}. 
As we show in the Appendix following arguments akin to \cite{benaim2005stochastic}, the limit set of any bounded APT of \eqref{eq:diffInclusion} is internally chain transitive.

Next we seek to characterize any stationarity guarantees for the points in this limit set. To this end: let $A$ be any subset of $\CM$. A continuous function $V\,:\,\CM\to\BR$ is called a \emph{Lyapunov function} for $A$ if $V(y)<V(x)$ for all $x\in\CM/A$, $y\in \Phi_t(x)$, $t>0$, and $V(y)\leq V(x)$ for all $x\in A$, $y\in\Phi_t(x)$, $t\geq0$.

\paragraph{Convergence Results}

From \cite{davis2020stochastic} Lemma 5.2 we see that if $f$ is a definable potential of a conservative field $D$ then Corollary \ref{th:chain} together with Theorem \ref{th:sard} implies that $f$ is a Lyapunov function for the differential inclusion $\dot \gamma(t)\in -\iota(\text{conv}(D_f(\gamma(t))))$ and we get the following generalization of Corollary 5.9 in \cite{davis2020stochastic}:
\begin{theorem}[Convergence]\label{th:convergence}
    Let $f:\CM\to\BR$ be a locally Lipschitz $C^k$-stratifiable function. Consider the iterates $\{x_k\}_{k\geq 1}$ produced by the process \eqref{eq:retsgddimsz} and assume \ref{assumptionsincl} with $G=-\iota(\text{conv}(D_f))$. Then, almost surely every limit point of the iterates $\{x_k\}_{k\geq 1}$ is critical for $f$ and the function values $\{f(x_k)\}_{k\geq 1}$ converges. 
\end{theorem}

The proof of this is in Appendix \ref{sec:appthmhdim}.




\subsection{Stochastic Approximation on a Manifold - Constant Stepsize}

We denote the kernel of the process defined by the SGD sequence~\eqref{eq:sgdseq} as $K_\alpha:\CM\times\mathcal{B}(\CM)\to[0,1]$. This acts on a measurable function $g:\CM\to\BR_+$ by
\begin{equation}
    K_\alpha g(x)=\int g(\exp_x\left[-\alpha\phi(x,s)\right])\mu(\rd s).
\end{equation}
Denote by $\Gamma$ then the set of all steps $\alpha$ such that $K_\alpha$ maps $\mathcal{P}_{\text{abs.}}(\CM)$ into itself:
\begin{equation}
    \Gamma\coloneqq \{ \alpha\in(0,\infty)\,:\,\forall\rho\in\mathcal{P}_{\text{abs.}},\,\rho K_\alpha\ll\lambda\}.
\end{equation}
When we consider tame functions, because of the non-smooth Morse-Sard theorem and the $C^k$ stratifications, $\Gamma^c$ is Lebesgue negligible, and in particular the closure of $\Gamma$ will contain $0$.

\begin{theorem}\label{th:tight}
    Let the standing assumptions, \ref{assum_f}, \ref{assum_tameness}, hold true. Consider $\alpha\in\Gamma$ and $\nu\in\mathcal{P}_{\text{abs.}}(\CM)\cap\mathcal{P}_1(\CM)$. Let $\{x_k\}_{k\in\BN^*}$ be an SGD sequence for $f$ with steps $\alpha$. Then, for any $k\in\BN$, it holds $\mathbb{P}^\nu$-a.e. that 
    \begin{enumerate}
        \item $F$, $f(\cdot,\xi_{k+1})$ and $f(\cdot,s)$ (for $\mu$-a.e. $s$) are differentiable at $x_k$, with $F$ as above;
        \item $x_{k+1}=\exp_{x_k}\left[-\alpha\grad f(x_k,\xi_{k+1})\right]$;
        \item $\mathbb{E}_k[x_{k+1}]=\exp_{x_k}\left[-\alpha\grad F(x_k)\right]$.
    \end{enumerate}
\end{theorem}

We now have the equivalent statement of~\cite[Theorem 2]{bianchi2022convergence}.
\begin{theorem}\label{th:weakconv}
Under the standing assumptions, \ref{assum_f} and \ref{assum_tameness}, let $\{(x^{\alpha}_k)_{k\in\BN^*}:\alpha\in(0,\alpha_0]\}$ be a collection of SGD sequences of steps $\alpha\in(0,\alpha_0]$. Define $x^{\alpha}$ iteratively to be:
\[
x^{\alpha}(t) = \gamma^k(t/\alpha-k),\,\forall t\in [k\alpha,(k+1)\alpha)
\]
where $\gamma^k:[0,1]\to \mathcal{M}$ is the geodesic curve with constant velocity $\|\dot{\gamma}^k\|$ from $\gamma^k(0)=x_k$ to $\gamma^k(1)=x_{k+1}$. 

It holds that for every compact set $\mathcal{K}\subset\mathcal{M}$,
\[
\forall \epsilon>0,\,\lim\limits_{\alpha\to 0,\alpha\in\Gamma}
\left(\sup\limits_{\nu\in\mathcal{P}_{abs}(\mathcal{K})} \mathbb{P}^{\nu}(\bd_C(x^{\alpha},\mathcal{S}_{-\partial F}(\mathcal{K})) > \epsilon)\right) = 0
\]
Moreover, the family of distributions $\{\mathbb{P}^{\nu}(x^{\alpha})^{-1}:\nu\in\mathcal{P}_{abs}(\mathcal{K}),0<\alpha<\alpha_0,\alpha\in\Gamma\}$ is tight.
\end{theorem}
The proof is given in the Appendix.



Next, we will further introduce the following assumptions. 
\begin{assumption}[Cf. Prop. 6 in \cite{bianchi2022convergence}]\label{assump:prop6}
    There exists a base point $0_\CM\in\CM$, constants $R>0$, $\varepsilon>0$, $C>0$ and a measurable function $\beta:\Omega\to\BR_+$ satisfying the following:
    \begin{enumerate}
        \item For every $s\in\Omega$, the function $f(\cdot,s)$ is differentiable outside of the set $\text{cl}(B(0_\CM,R))$. Moreover, for each $x, x'\notin \text{cl}(B(0_\CM,R))$ we have $\norm{P_{xx'}\grad f(x)-\grad f(x')} \leq \beta(s)\bd(x,x')$ and $\int \beta^2\rd\mu<\infty$. 
        \item For all $x\notin \text{cl}(B(0_\CM,R))$, we have $\int\norm{\grad f(x,s)}^2\mu(\rd s)\leq C(1+\norm{\grad F(x)}^2)$.
        \item $\lim_{\bd(x,0_\CM)\to\infty}\norm{\grad F(x)}=+\infty$.
        \item The function $F$ is lower bounded.
    \end{enumerate}
\end{assumption}

As discussed in \cite{bianchi2022convergence}, and as we show in the Appendix, these assumptions, together with the standing assumptions \ref{assum_f}, implies further that:
\begin{proposition}[Cf. assumption 4 of \cite{bianchi2022convergence}]\label{prop:Assumption4}
    There exists measurable functions $V:\CM\to [0,+\infty)$, $p:\CM\to[0,+\infty)$, $\beta:(0,+\infty)\to (0,+\infty)$ and a constant $C\geq 0$ such that the following holds for every $\alpha\in\Gamma\cap(0,\alpha_0]$.
    \begin{enumerate}
        \item There exists $R>0$ and a positive Borel measure $\rho$ on $\CM$ such that 
        \begin{equation}
            \forall x\in \text{cl}(B(0_\CM,R)),\, \forall A\in\mathcal{B}(\CM),\,K_\alpha(x,A)\geq\rho(A). 
        \end{equation}

        \item $\sup_{\text{cl}(B(0_\CM,R))}V<\infty$ and $\inf_{B(0_\CM,R)^c}p>0$. Moreover, for every $x\in\CM$
        \begin{equation}
            K_\alpha V(x)\leq V(x)-\beta(\alpha)p(x)+C\beta(\alpha)\textbf{1}_{\bd(x,0_\CM)\leq R}.
        \end{equation}

        \item The function $p(x)$ diverges to infinity as $\bd(x,0_\CM)\to\infty$.
    \end{enumerate}
\end{proposition}

We now finally have the following result.

\begin{theorem}[Convergence -- constant step size]\label{th:convCSS}
    Let the Assumptions \ref{assum_f}, \ref{assum_tameness} and \ref{assump:prop6} hold true. Let $\{(x_n^\alpha)_{n\in\BN}:\alpha\in(0,\alpha_0]\}$ be a collection of SGD sequences of step size $\alpha$. Then, the set $\mathcal{Z}\coloneqq \{x:0\in\partial F(x)\}$ is nonempty and for all $\nu\in\mathcal{P}(\CM)$ and all $\epsilon>0$,
    \begin{equation}
        \limsup_{n\to\infty}\mathbb{P}^\nu\left(\textbf{d}(x_n^\alpha,\mathcal{Z})>\epsilon\right)\Longrightarrow_{\alpha\to0,\alpha\in\Gamma}0.
    \end{equation}
\end{theorem}

\section{Numerical Results}

We present the results on the numerical performance of Riemannian stochastic approximation on a set of standard representative examples of nonsmooth data-fitting objective functions. We are going to report the results of performing the Retracted-SGD using three step-size rules: 
\begin{itemize}
    \item Diminishing step size Regime 1 : $s_0,~s_{k+1}=(1-cs_k)s_k$.
    \item Diminishing step size Regime 2 : $s_k=a/(1+k)^b$.
    \item Constant step size $s_k\equiv s$
\end{itemize}  
such that $a>0,1/2<b\le 1, s_0>0$ and $0<c<1$ are constant parameters and $k$ is the current iterate. 

For the experiments, we used an Asus notebook computer with a dedicated NVIDIA GeForce GT 820M graphics card with 1GB of GDDR3 discrete memory and a 6GB memory Corei7 processor using the \texttt{Pytorch} and \texttt{Pymanopt} libraries. Hyperparameters were optimized for the diminishing stepsize parameters using \texttt{optuna}. We performed $5$ runs for the experiments such that each run consists of $10^4$ iterations.

\subsection{Sparse PCA}
We seek the principal component vectors of a large-scale matrix $A$. The problem can be written formally as,
\begin{equation}\label{eq:sparsepca}
\begin{array}{rl} 
\min\limits_{X\in\mathcal{M}} & -\mathop{tr}(X^TA^T A X)+\rho \|X\|_1\\
& \mathcal{M}:=\{ X\in\mathbb{R}^{n\times p},\, X^T X = I_p\}
\end{array}
\end{equation}
To consider the problem as stochastic, at each iteration, we sample a subset of rows of $A$, i.e., 
\[
A = \mathbb{E}[A(\xi)]= n \begin{pmatrix} \mathbf{1}_{p}(1)a_1 \\ \mathbf{1}_{p}(2) a_2 \\ \cdots \end{pmatrix} 
\]
where with probability $1/n$ we sample $p\in [n]$.

We display the results for a random matrix $A$ of order $100\times100$. First, Let $s_0=1$, $a$=\texttt{10}, and different constant step sizes, for fixed values $c=$\texttt{1e-4} and $b=1$. In Figure \ref{Fig:SparsePCA0} we can see the monotonic decrease in the loss of the objective function.
\begin{figure}[h!] 
    \centering
    \includegraphics[scale=0.35]{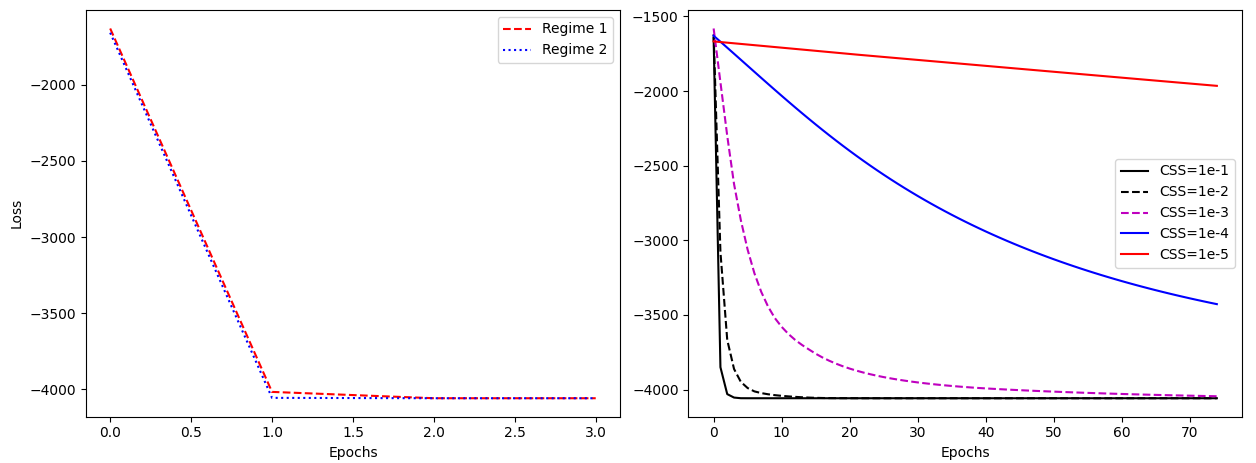}
    \caption{Diminishing~Step~Size,~$s_0 \leq\leq a$, and Constant Step Sizes.}
    \label{Fig:SparsePCA0}
\end{figure}

For numerical convergence of the methods, observe Figure \ref{Fig:SparsePCA1}: as the number of iterates ($x$-axis) increases the norm of the retracted gradients ($y$-axis) decreases to zero. Both methods are asymptotically convergent, with variable rates. 
\begin{itemize}
    \item Figure \ref{Fig:SparsePCA1} displays the norm of the computed retracted gradient in both methods. Regime 2 outperforms Regime 1 in the diminishing step size, and a large stepsize is warranted for this problem in the case of a constant stepsize.
\begin{figure}[h!]
    \centering
    \includegraphics[scale=0.4]{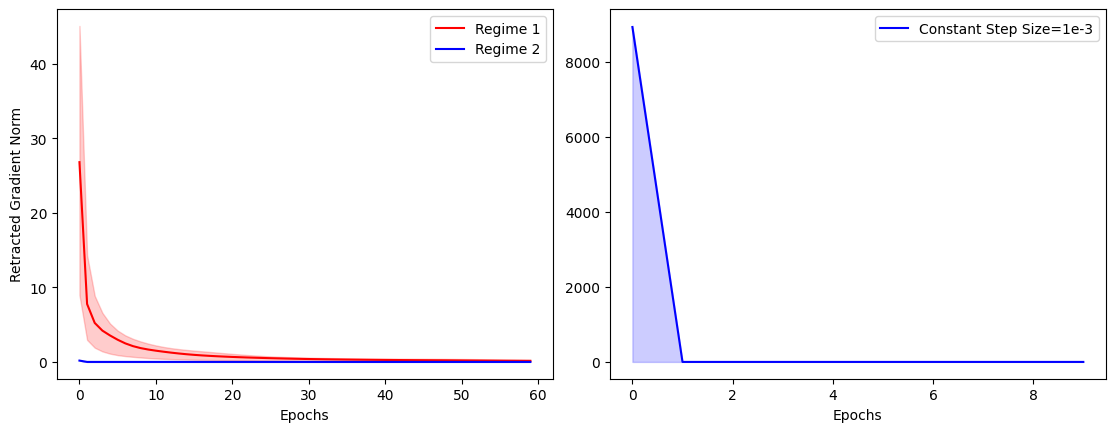}
    \caption{Diminishing~Step~Size,~$s_0 \leq\leq a$, and Constant Step Size.}
    \label{Fig:SparsePCA1}
\end{figure}
\item In Figure \ref{Fig:SparsePCA2}, we change $s_0$=\texttt{1e3}, $a$=\texttt{0.2}, for fixed values $c=$\texttt{1e-4} and $b=1$, we can see the differences compared to figure \ref{Fig:SparsePCA1}
\begin{figure}[h!]
    \centering
    \includegraphics[scale=0.5]{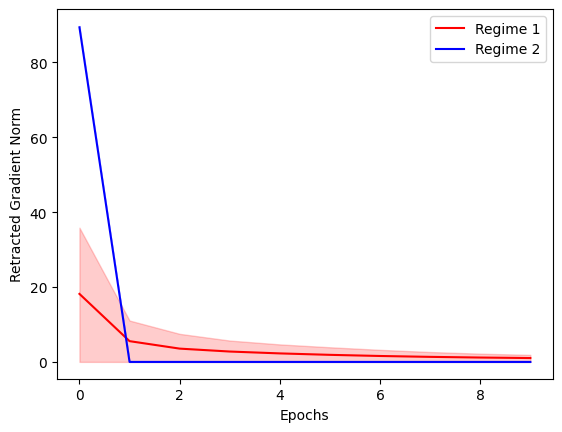}
    \caption{Changed Initial Values (Diminishing Step Size $s_0 \geq\geq a$)}
    \label{Fig:SparsePCA2}
\end{figure}
\item In Figure \ref{fig:SparsePCA3},  by considering a reduced dimension, we can see the empirical distribution of trajectories on the Stiefel manifold (p=1), for constant step sizes $1e-1$, $1e-2$ and $1e-4$, starting from the same initial point. We observe that the limiting behavior, as illustrated by the empirical distribution of end points, generally resembles the central limit ergodic stationary distribution with variance proportional to the stepsize.
\begin{figure}
    \centering
\includegraphics[scale=0.2]{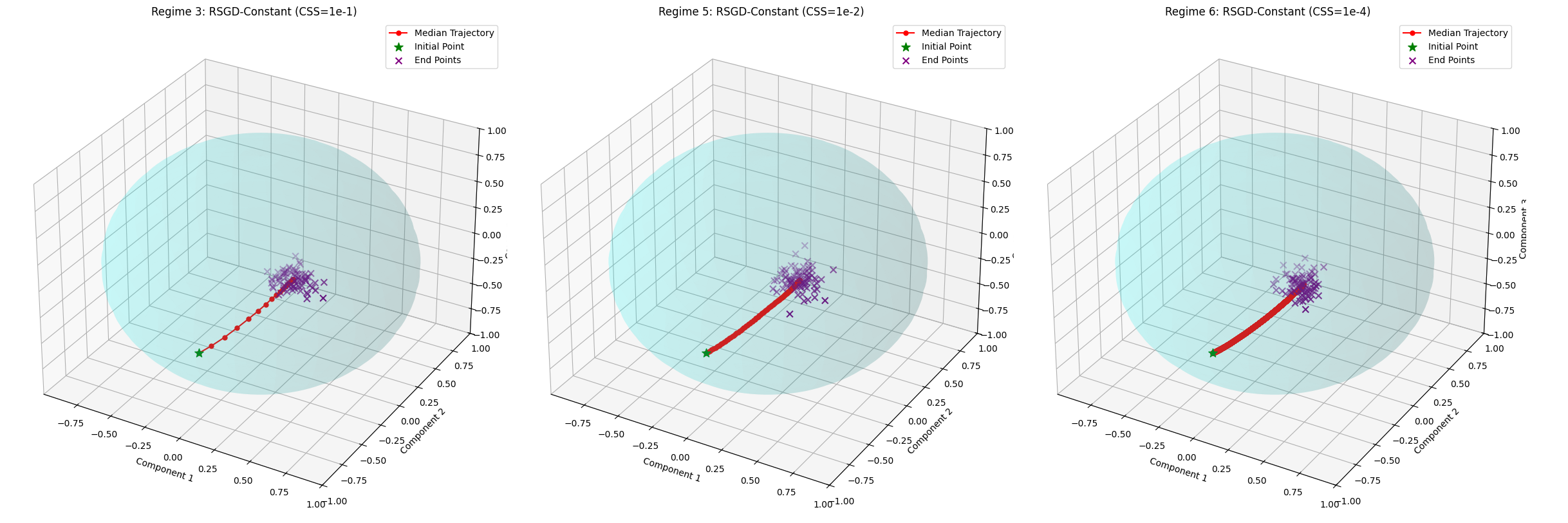}
    \caption{Distribution on The Stiefel Manifold}
    \label{fig:SparsePCA3}
\end{figure}

\end{itemize}
\newpage

\subsection{Low Rank Matrix Completion}
The low-rank Matrix Completion problem can be written as,
\begin{equation}\label{eq:matrixcompletion}
\begin{array}{rl} 
\min\limits_{X\in\mathcal{M}} & \sum\limits_{i,j}|A_{ij}-X_{ij}|\\
& \mathcal{M}:=\{ X\in\mathbb{R}^{m\times n},\, \mathop{rank}(X)=p\}
\end{array}
\end{equation}
To make the problem stochastic we consider that access to the components of $A$ is noisy, i.e., every evaluation of $A_{ij}$ is of the form $A_{ij}+\epsilon$ where $\epsilon\sim \mathcal{N}(0,\sigma)$ for a small $\sigma$.

We report the results of using Retracted-SGD with diminishing and constant step sizes. We consider a normalized random matrix of order $10^2\times 10^2$ with added noise $\epsilon$ where $\epsilon\in \mathcal{N}(0,\sigma I)$ and $\sigma=\texttt{1e-3}$.

In Figure \ref{Fig:LowRankMatrixCompletion0} we will show the decrease in the loss of objective function. Let $s_0=a$=\texttt{1e1}, and different constant step sizes, for fixed values $c=$\texttt{1e-4} and $b=1$.

\begin{figure}[h!]
    \centering
    \includegraphics[scale=0.4]{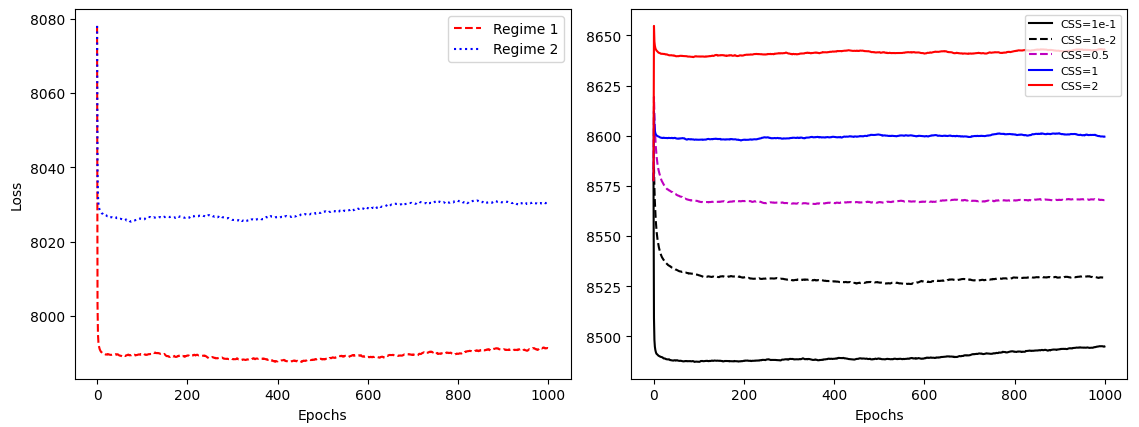}
    \caption{Diminishing Step Size, $s_0 \leq\leq a$, and Constant Step Sizes}
    \label{Fig:LowRankMatrixCompletion0}
\end{figure}

\begin{itemize}
    \item In Figure \ref{Fig:LowRankMatrixCompletion1}, we can see the value of the retracted gradient ($y$-axis) decreases as the number of iterations ($x$-axis) increases. 
\begin{figure}[h!]
    \centering
    \includegraphics[scale=0.4]{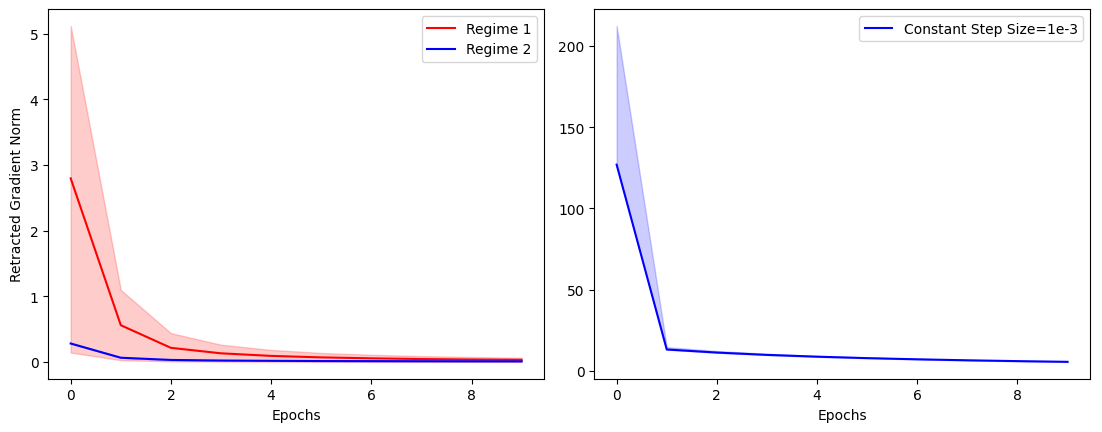}
    \caption{Diminishing Step Size, $s_0 \leq\leq a$, and Constant Step Sizes} 
    \label{Fig:LowRankMatrixCompletion1}
\end{figure}
\item In Figure \ref{Fig:LowRankMatrixCompletion2}, for diminishing step size, we set $s_0$=\texttt{1e2}, $a$=\texttt{1e-1} and fixed values $c=$\texttt{1e-4} and $b=1$. In comparison to Figure \ref{Fig:LowRankMatrixCompletion1}, regime 1 performs better than regime 2. We observe a rapid convergence for both regimes again.
\begin{figure}[h!]
    \centering
    \includegraphics[scale=0.5]{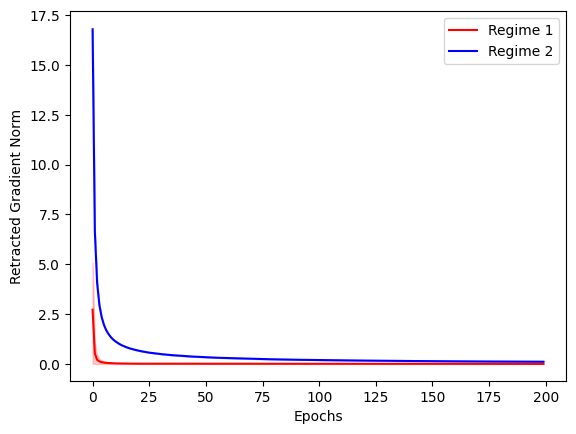}
    \caption{Changed Initial Values (Diminishing Step Size $s_0 \geq\geq a$)}
    \label{Fig:LowRankMatrixCompletion2}
\end{figure}
\end{itemize}

\newpage
\subsection{ReLU Neural Network with Batch Normalization}
Inspired from~\cite{hu2020brief}, we shall also consider the problem of training a neural network with batch normalization. We consider regression with a network composed of Rectified Linear Units (ReLUs), activation functions that take a component-wise maximum between the previous layer, linearly transformed, and zero. For additional nonsmoothness, we consider an $l1$ loss function. Batch Normalization amounts to scaling the weights at each layer to encourage stability in training. Formally,
\begin{equation}\label{eq:relubatchnorm}
\begin{array}{rl} 
\min\limits_{w\in\mathcal{M}} & \frac{1}{N}\sum\limits_{i=1}^N|\hat{y}(x_i;w)-y_i|\\
& \mathcal{M}:=\{ x\in\mathbb{S}^{n_1}\times \mathbb{S}^{n_2}\times\cdots \times \mathbb{S}^{n_L}\times \mathbb{R}^{n_o}\}
\end{array}
\end{equation}
where we have $N$ training examples $\{x_i,y_i\}$ and $\hat{y}$ is the neural network
model given the set of weights $w$, at each layer $j=1,...,L$, there are $n_j$ weights to be normalized, and there are a $n_0$ additional unnormalized weights.

We are reporting our results for $L=3$, therefore we will have a $3$-layers neural network. The number of nodes in the hidden layers is of order $3\times 3$. We have a LIBSVM regression data set,
 \texttt{"bodyfat"} in which $N=252$ and the number of features is $14$. 

In the Figure \ref{Fig:3 data sets-BatchNormalization (1)}, for a random batch size = $64$, we try to set optimal values, respectively,
$s_0=a\approx 0.052,~c\approx 0.00325$, $ b\approx 0.7$ and the constant step sizes as in the previous problems from 1e-1 to 1e-5. The loss of the neural network ($y$-axis) decreases as the number of iterates ($x$-axis) increases. We have the convergence in both methods.
\begin{center}
\includegraphics[scale=0.4]{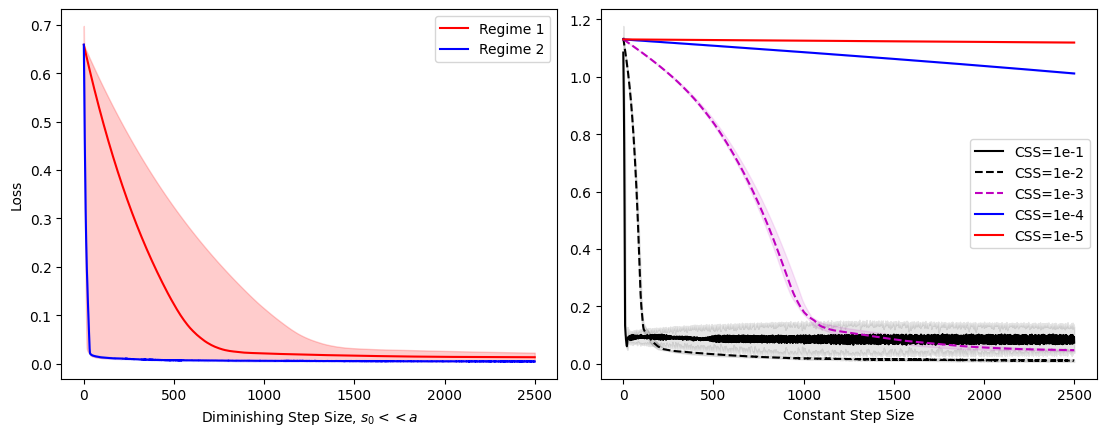}
\captionof{figure}{Methods Convergence and The Loss of Objective Function }
\label{Fig:3 data sets-BatchNormalization (1)}
\end{center}

The method's efficiency is sensitive to initial parameter values. In Figure \ref{Fig:3 data sets-BatchNormalization (2)}, we have again batch size = 64,  but nonoptimal values $s_0=\texttt{1e1},~a=\texttt{1e3}, ~c=\texttt{1e-4}$,  $b=1$, for diminishing step size. One can see that the methods depend on parameters by comparing  \ref{Fig:3 data sets-BatchNormalization (1)}.
\begin{center}
\includegraphics[scale=0.5]{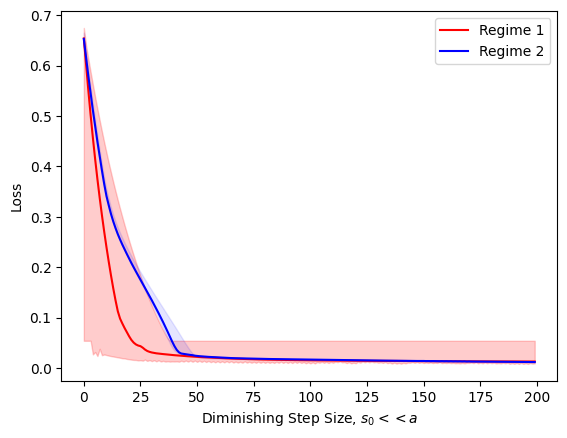}
\captionof{figure}{The Regimes' Sensitivity To The Parameters' Initial Values.}
\label{Fig:3 data sets-BatchNormalization (2)}
\end{center}

\section{Discussion and Conclusions}
We presented a foundational development of nonsmooth stochastic approximation for problems constrained on Riemannian manifolds. We presented background on Riemannian geometry, Riemannian probability and measure theory, tame functions, and SGD sequences on manifolds. There are a number of intuitive algorithmic features that are present in the Euclidean case that are not transferable to manifolds. We hope that the readers have appreciated as us the rich interplay of multiple deep mathematical domains in the problem

We presented, and proved in the Appendix, convergence to almost sure solutions of the asymptotic pseudotrajectory defined by the iterates as well as ergodicity and weak convergence of trajectories for constant stepsize SGD. We hope that the work will encourage the development of new methods, possibly using the manifold structure to improve convergence properties. The proof techniques illustrate the key important adjustment between the Euclidean and Riemannian case: 1) In the former, (sub)gradient directions and vector differences between points lie in the same space, 2) in the latter, the corresponding cotangent bundles and geodesic paths are completely distinct. 

The numerical results indicated no fundamental problem as far as seeing the usual convergence properties of stochastic subgradient methods. Future work could be to understand the interplay between the manifold curvature and convergence properties, investigate the consequences of challenging problem assumptions, and the consideration of more adaptive algorithms as well as other contemporary ML methods (e.g., momentum methods and Adam).

\bibliographystyle{plain}
\bibliography{refs.bib}

\section{Appendix: Proofs}
In this Appendix, we provide the proofs of all the results in the main text.

\subsection{Preliminary Results}\label{app:proof_prel}
The proofs of Theorems \ref{th:chain}, \ref{th:VarStrat} and \ref{th:sard} are trivially lifted from the references to the Riemannian setting, and we do not include them here.

\begin{proof}\textbf{of Theorem} \ref{th:gradEverywhere}:
    We follow the procedure in \cite{bolte2021conservative}. Fix a measurable selection $a:\CM\to T^*\CM$ of $D$ and a potential $f$ of $D$. Further, fix a point $x\in\CM$, a vector $v\in T_x\CM$ and let $\gamma(t)\,:\,[0,1]\to\CM$ be the curved defined by $\gamma(0)=x$ and $\dot\gamma(0)=v$. 

    The definition of a potential function gives us
    \begin{equation}
        f(\gamma(1))-f(x)=\int_0^1\langle a(\gamma(t)),\dot\gamma(t)\rangle\rd t,
    \end{equation}
    giving us $\rd f_x(v)=\lim_{t\to 0}\frac{f(\gamma(t))-f(x)}{t} =\langle a,v\rangle$ almost everywhere along $\gamma(t)$, since Rademacher's theorem, see \cite[Thm. 5.7]{azagra2005nonsmooth}, says that $f$ is differentiable almost everywhere. 


    Introduce the Dini derivatives \cite{ferreira2008dini}
    \begin{equation}
        \begin{aligned}
            &f'_u(y;v)\coloneqq \limsup_{t\to 0^+}\frac{f(\gamma(t))-f(x)}{t},\\
            &f'_l(y;v)\coloneqq \liminf_{t\to 0^+}\frac{f(\gamma(t))-f(x)}{t}.
        \end{aligned}
    \end{equation}
    Since $f$ is measurable, $f'_u$ and $f'_l$ are as well. Now, consider the set
    \begin{equation}
        A=\{y\in\CM\,:\, f'_u(y;v)\neq \langle a(y),v\rangle\text{ or } f'_l(y;v)\neq \langle a(y),v\rangle\}.
    \end{equation}
    This set is measurable and for $y\in\CM\backslash A$ and $v\in T_y\CM$ we have
    \begin{equation}
        \rd f_y(v)=f'_u(y;v)=f'_l(y;v)=\langle a(y),v\rangle.
    \end{equation}
    Furthermore, $\lambda(A\cap \gamma(I))=0$, with $I=[0,1]$. Fubini's theorem then tells us that $\lambda(A)=0$. Both $x$ and $v$ were chosen at random and we can run the same argument for any $x\in\CM$ and any $v\in T_x\CM$, then we see that $\rd f_y=a(y)$ for almost all $y\in\CM$.

    Furthermore, the selection $a$ was chosen arbitrarily and Corollary 18.15 of \cite{aliprantisborder} then tells us that there exists a sequence of measurable selections $(a_k)_{k\in\BN}$ of $D$ such that for any $x\in\CM$ $D(x)=\overline{\{a_k(x)\}}_{k\in\BN}$. Rademacher again tells us that there exists a sequence of measurable sets $(S_k)_{k\in\BN}$ with full measure and such that $a_k=\rd f$ on each $S_k$. Setting $S=\bigcap_{k\in\BN}S_k$ we have that $\CM\backslash S$ has zero measure and this implies that $D=\{\rd f\}$ on $S$. 
    
\end{proof}

\subsection{Proof of Theorem \ref{th:convergence}}\label{sec:appthmhdim}

For later reference, let us state the following definitions.

\begin{definition}[Attracting set, \cite{benaim2005stochastic}]
    Given a closed invariant set $L$, the induced set-valued dynamical system $\Phi^L$ is the family of set-valued mappings $\Phi_L=\{\Phi_t^L\}_{t\in\BR}$ defined by
    \begin{equation}
        \begin{aligned}
            \Phi_t^L(x)\coloneqq \{& \gamma(t)\,:\, \gamma(t) \text{ is a solution of \eqref{eq:diffInclusion} with }\\
            &\gamma(0)=x \text{ and } \gamma(\BR)\subset L\}.
        \end{aligned} 
    \end{equation}
    A compact subset $A\subset L$ is called an \emph{attracting set} for $\Phi^L$, provided that there is a neighborhood $U$ of $A$ in $L$ with the property that for any $\varepsilon>0$ there exists $t_\varepsilon>0$ such that
    \begin{equation}
        \Phi_t^L(U)\subset \bigcup_{x\in A}B(x,\varepsilon),\quad \forall t>t_\varepsilon.
    \end{equation}
    If $A$ is an invariant set, then $A$ is called an \emph{attractor} for $\Phi^L$. Note that, an attracting set (or attractor) for $\Phi$ is an attracting set (or attractor) for $\Phi^L$ with $L=\CM$. If $A\neq L,\,\varnothing$, then $A$ is called a \emph{proper} attracting set (or attractor) for $\Phi^L$. 
    Finally, the set $U$ is referred to as a \emph{fundamental neighborhood} of $A$ for $\Phi^L$.
\end{definition}

\begin{definition}[Asymptotic stability, \cite{benaim2005stochastic}]
    A set $A\subset L$ is called \emph{asymptotically stable} for $\Phi^L$ if it satisfies the following:
    \begin{enumerate}
        \item $A$ is invariant;
        \item $A$ is Lyapunov stable, meaning that for every neighborhood $U$ of $A$ there exists a neighborhood $V$ of $A$ such that $\Phi_{[0,\infty)}(V)\subset U$;
        \item $A$ is attractive, i.e., there exists a neighborhood $U$ of $A$, such that for any $x\in U$ we have $\omega_\Phi(x)\subset A$. 
    \end{enumerate}
\end{definition}

The set
\begin{equation}
    \omega_\Phi(x)\coloneqq \bigcap_{t\geq 0}\overline{\Phi_{[t,\infty)}(x)}
\end{equation}
is the $\omega$-limit set of a point $x\in\CM$.  

We now have the following results from \cite{benaim2005stochastic}:
    \begin{theorem}[Cf. Theorem 4.1 in \cite{benaim2005stochastic}]\label{th:APT}
    Assume $\zeta$ is bounded. If $\zeta$ is uniformly continuous and any limit point of $\{\Theta^t(\zeta)\}$ is in $\mathcal{S}$, then $\zeta$ is an APT for $\Phi$. 

    Here, the set $\{\Theta^t,\,t\geq 0\}$ is relatively compact. 
\end{theorem}

\begin{proof}
    The proof follows \cite{benaim2005stochastic}. If $\zeta$ is uniformly continuous, the family of functions $\{ \Theta^t(\zeta)\,:\,t\geq 0\}$ is equicontinuous and the Arzel\`a-Ascoli theorem tells us that the family is relatively compact, and \eqref{eq:APTdef} follows. 
    

\end{proof}

\begin{theorem}[Theorem 4.2 in \cite{benaim2005stochastic}]\label{th:pert-APT}
    Any bounded perturbed solution $\tilde\gamma$ is an APT of the differential inclusion \eqref{eq:diffInclusion}.
\end{theorem}

\begin{proof}
    We prove that $\tilde\gamma$ satisfies the assumptions of Thm. \ref{th:APT}. To this end, set $v(t,s)=P^{u_t}_{\tilde\gamma(t)u_t(s)}\dot{\tilde\gamma}(t)-\dot u_t(s)\in G^{\delta(t)}(\tilde\gamma(t))$. Then, 
    \begin{equation}\label{eq:t1plust2}
        \begin{aligned}
            &\bd(\tilde\gamma(t_1+t_2),\tilde\gamma(t_1))\\
            =& \int_{T_{t_1}}^{T_{t_1+t_2}}\int_0^{t_2} \norm{v(t_2+\tau,s)}_{g_{u_{t_2+\tau}(s)}}\rd \tau\rd s\\
            &+\int_{T_{t_1}}^{T_{t_1+t_2}}\int_{t_1}^{t_1+t_2}\norm{\dot u_\tau(s)}_{g_{u_\tau(s)}}\rd \tau\rd s.
        \end{aligned}
    \end{equation}
    By the defining properties of $u_t$ the second integral goes to zero as $t\to\infty$. The boundedness of $\tilde\gamma$, $\tilde\gamma(\BR)\in A$ for compact $A$ then implies the boundedness of $v$ and we thus see that $\tilde\gamma$ is uniformly continuous. Thus, $\Theta^t(\tilde\gamma)$ is equicontinuous, and relatively compact. Let $\zeta=\lim_{n\to\infty}\Theta^{t_n}(\tilde\gamma)$ be a limit point, set $t_1=r_n$ above \eqref{eq:t1plust2} and define $v_n(t_2,s)=v(r_n+t_2,s)$. Again, by the defining properties of $u_t$ the second integral in \eqref{eq:t1plust2} will vanish uniformly when $n\to\infty$. Hence, 
    \begin{equation}
        \bd(\zeta(t_2),\zeta(0))=\lim_{n\to\infty}\int_{T_{r_n}}^{T_{r_n}+T_{t_2}} \int_0^{t_2}\norm{v_n(\tau)}_{g_{u_{r_n+\tau}(s)}}\rd\tau\rd s.
    \end{equation}
\end{proof}
Since $(v_n)$ is uniformly bounded, Banach-Alaoglu tells us that a subsequence of $v_n$ will converge weakly in $L^2[0,t_2]$ to some function $v$ with $v(t)\in G(\zeta(t))$, for almost any $t$, since $v_n(t)\in G^{\delta(t+r_n)}(\tilde \gamma(t+r_n))$ for every $t$. Now, a convex combination of $\{v_m,\,m>n\}$ converges almost surely to $v$, by Mazur's lemma, and $\lim_{m\to\infty} \text{conv}\left(\bigcup_{n\geq m} G^{\delta(t+r_n)}(\tilde\gamma(t+r_n))\right)\subset G(\zeta(t))$. We thus have that $\bd(\zeta(t_2),\zeta(0))=\int_{T_{t_1}}^{T_{t_1}+T_{t_2}}\int_0^{t_2}\norm{v(\tau,s)}_{g_{u_\tau}(s)} \rd\tau\rd s$. This proves that $\zeta$ is a solution of \eqref{eq:diffInclusion} and that $\zeta\in \widehat{\mathcal{S}}_A$.

\begin{theorem}[Theorem 4.3 in \cite{benaim2005stochastic}]\label{th:ict}
    Let $\zeta$ be a bounded APT of \eqref{eq:diffInclusion}. Then $L(\zeta)$ is internally chain transitive. 
\end{theorem}

\begin{proof}
    The proof is more or less identical to that of \cite{benaim2005stochastic}. The set $\{\Theta^t(\zeta)\,:\,t\geq 0\}$ is relatively compact, and the $\omega$-limit set of $\zeta$ for the flow $\Theta$,
    \begin{equation}
        \omega_\Theta(\zeta)=\bigcap_{t\geq 0}\overline{\{\Theta^s(\zeta)\,:\,s\geq t\}},
    \end{equation}
    is therefore internally chain transitive. From \eqref{eq:APTdef} we know that $\omega_\Theta(\zeta)\subset \mathcal{S}$.

    Let $\Pi\,:\,(C^0(\BR,\CM),\bd_C)\to(\CM,\bd(\cdot,\cdot))$ be the projection map defined by $\Pi(\zeta)=\zeta(0)$. This gives $\Pi(\omega_\Theta(\zeta))=L(\zeta)$. Now, set $p=\lim_{n\to\infty}\zeta(t_n)$ and $w$ a limit point of $\Theta^{t_n}(\zeta)$, then $w\in \omega_\Theta(\zeta)$ and $\Pi(w)=p$. This implies that $L(\zeta)$ is nonempty, compact and invariant under $\Phi$, since $\omega_\Theta(\zeta)\subset \mathcal{S}$. The projection $\Pi$ has Lipschitz constant one and maps every $(\varepsilon,T)$ chain for $\Theta$ to an $(\varepsilon,T)$ chain for $\Phi$. This means that $L(\zeta)$ is internally chain transitive. 
\end{proof}

We present the following theorem, which in turn will directly imply Theorem \ref{th:convergence}. The proof of this Theorem follows that of Theorem 3.2 in \cite{davis2020stochastic}.
\begin{theorem}\label{th:3.2}
    Suppose that assumption \ref{assumptionsincl} holds and that there exists a Lyapunov function $\varphi$ on $\CM$. Then every limit point of $\{x_k\}_{x\geq 1}$ lies in $G^{-1}(0)$ and the function values $\{\varphi(x_k)\}_{k\geq 1}$ converge.
\end{theorem}

However, this proof has several components composed of technical Lemmas that we adapt for the sake of completeness.  
\begin{lemma}
    $\lim_{k\to\infty}\bd(x_k,x_{k+1})=0$
\end{lemma}

\begin{proof}
    The retraction is a smooth map from $T\CM\to \CM$ and we have $D\mathcal{R}_x(0_x)[v]=v$. Furthermore, we can define a curve $\gamma(t)=\mathcal{R}_x(tv)$, $\dot\gamma(0)=v$ and $\mathcal{R}_x(0_x)=x$. Now, 
    \begin{equation}
        \norm{v}_{g_x}\leq \alpha_k\max_{z\in F(x_k)}\norm{z}_{g_x}+\alpha_k\norm{\dot u_t}_{g_x}.
    \end{equation}
    But the second part goes to zero by assumption and $\norm{z}_{g_x}$ is bounded while $\alpha_k\to 0$, so this implies that $v$ converges to zero and therefore $\bd(x_k,x_{k+1})\to 0$, since $\mathcal{R}(tv)$ is smooth. 
\end{proof}

\begin{lemma}
    We have
    \begin{equation}
        \liminf_{t\to\infty}\varphi(x(t))=\liminf_{k\to\infty}\varphi(x_k),
    \end{equation}
    and
    \begin{equation}
\limsup_{t\to\infty}\varphi(x(t))=\limsup_{k\to\infty}\varphi(x_k).
    \end{equation}
\end{lemma}

\begin{proof}
    Clearly, we have that $\leq$ and $\geq$ holds, respectively, in the two equations. To argue for the opposite direction we let $\tau_i\to\infty$ be an arbitrary sequence with $x(\tau_i)$ converging to some point $x^*$ as $i\to\infty$. For each index $i$, we define the breakpoint $k_i=\max\{k\in\BN\,:\,t_k\leq \tau_i\}$. Then
    \begin{equation}
        \begin{aligned}
            &\bd(x_{k_i},x^*)\leq \bd(x_{k_i},x(\tau_i))+\bd(x(\tau_i),x^*)\\
            &\leq \bd(x(\tau_i),x^*)+\bd(x_{k_i},x_{k_i+1}).
        \end{aligned}
    \end{equation}
    Since the right hand side tends to zero we get that $x_{k_i}\to x^*$, which further implies that $\varphi(x_{k_i})\to\varphi(x^*)$. 

    Let now $\tau_i\to\infty$ be a sequence realizing $\liminf_{t\to\infty}\varphi(x(t))$. Since $x(t)$ is bounded, $x(\tau_i)$ converges to some point $x^*$ and we find
    \begin{equation}
        \liminf_{k\to\infty}\varphi(x_k)\leq \lim_{i\to\infty}\varphi(x_{k_i})=\varphi(x^*)=\liminf_{t\to\infty}\varphi(x(t)).
    \end{equation}
    The second equality can be shown in the same way. 
\end{proof}

\begin{proposition}
    The values $\varphi(x(t))$ have a limit as $t\to\infty$.
\end{proposition}

\begin{proof}
    Without loss of generality, suppose $0=\liminf_{t\to\infty}\varphi(x(t))$. For each $r\in\BR$, define the sublevel set
    \begin{equation}
        L_r\coloneqq \{x\in\CM\,:\,\varphi(x)\leq r\}.
    \end{equation}
    Choose any $\varepsilon>0$ satisfying $\varepsilon\notin \varphi(G^{-1}(0))$. Note that $\varepsilon$ can be arbitrarily small. According to the above lemma we then have infinitely many indices $k$ such that $\varphi(x_k)<\varepsilon$. Then, for sufficiently large $k\in\BN$ we have
    \begin{equation}
        x_k\in L_\varepsilon\implies x_{k+1}\in L_{2\varepsilon}.
    \end{equation}
    This follows from the same argument as in \cite{davis2020stochastic}. 

    Let us define a sequence of iterates. Set $i_1\in \BN$ as the first index satisfying
    \begin{enumerate}
        \item $x_{i_1}\in L_\varepsilon$;
        \item $x_{i+1}\in L_{2\varepsilon}\backslash L_\varepsilon$;
        \item defining the exit time $e_1=\min\{e\geq i_1\,:\, x_e\notin L_{2\varepsilon}\backslash L_\varepsilon\}$, the iterate $x_{e_1}$ lies in $\CM\backslash L_{2\varepsilon}$.
    \end{enumerate}
    Next, let $i_2>i_1$ be the next smallest index satisfying the same properties, and so on. This process must then terminate, i.e., $\{x_k\}$ exits $L_{2\varepsilon}$ a finite amount of times. Then we see that the proposition follows, since $\varepsilon$ can be made arbitrarily small and the above lemma gives us $\lim_{t\to\infty}\varphi(x(t))=0$. 
\end{proof}

Now we can give the proof of the theorem. 

\begin{proof}[Proof of Theorem \ref{th:3.2}]
    Let $x^*$ be a limit point of $\{x_k\}$ and suppose for the sake of contradiction that $0\notin G(x^*)$. Let $i_j$ be indices satisfying $x_{i_j}\to x^*$ as $j\to\infty$. Let $z(\cdot)$ be the subsequential limit of the curves $\gamma^{t_{i_j}}(\cdot)$ in $C(\BR_+,\CM)$ which are guaranteed to exist. The existence of the Lyapunov function guarantees that there exists a real $T>0$ satisfying
    \begin{equation}
        \varphi(z(T))<\sup_{t\in[0,T]}\varphi(z(t))\leq \varphi(x^*).
    \end{equation}
    But, we can deduce
    \begin{equation}
        \varphi(z(T))=\lim_{j\to\infty}\varphi(\gamma^{t_{i_j}}(T))=\lim_{t\to\infty}\varphi(\gamma(t))=\varphi(x^*),
    \end{equation}
    where we used the above proposition and continuity of $\varphi$ in the last step. But this is a contradiction and the theorem follows. 
\end{proof}

\subsection{Proof of Theorem \ref{th:tight}}
\begin{proof}Take $\nu\ll\lambda$. We follow \cite{bianchi2022convergence} and prove this iteratively. 

To prove 1., we consider first the sets $S_1$ and $S_2$ consisting of the points $x\in\CM$ where $f(x,s)$ respectively $F(x)$ are differentiable. Assuming \ref{assum_f} and applying Rademacher's theorem together with Theorem \ref{th:gradEverywhere} and Fubini's theorem we see that $S_1\in\mathcal{B}(\CM)$ and $\lambda(\CM\backslash S_1)=0$. So $f(\cdot,s)$ is differentiable at $x_0$ for $\mu$-a.e. $s\in\Omega$, and since $\xi_{1}\sim \mu$, we also have that $f(\cdot,\xi_1)$ is differentiable at $x_0$. Similarly, Rademacher tells us that $\lambda(\CM\backslash S_2)=0$, such that also $F$ is differentiable at $x_0$.  

To prove 2. and 3., we first note that with probability one $x_0\in S_1\cap S_2$. Then, we define $A(x)\coloneqq \{s\in\Omega:(x,s)\notin \Delta_f\}$. For all $x\in S_1\cap S_2$ and all $v\in T^*_x\CM$, we have
\begin{equation}
    \begin{aligned}
        \langle\int \grad f(x,s)\mathbf{1}_{\Delta_f}(x,s)\mu(\rd s),v\rangle=&\int_{\Omega\backslash A(x)}\langle\grad f(x,s),v\rangle\mu(\rd s)\\
        =&\int_{\Omega\backslash A(x)}\lim_{t\to 0}\frac{f(\gamma(t),s)-f(x,s)}{t}\mu(\rd s)\\
        =&\lim_{t\to 0}\int_{\Omega}\frac{f(\gamma(t),s)-f(x,s)}{t}\mu(\rd s)\\
        =&\lim_{t\to 0}\frac{F(\gamma(t))-F(x)}{t}\\
        =&\langle \grad F(x),v\rangle,
    \end{aligned}
\end{equation}
where $\gamma(t):[0,1]\to\CM$ is the curve defined by $\gamma(0)=x$ and $\dot\gamma(0)=v$, and we used the dominated convergence theorem and assumption \ref{assum_f} in going from the second to third line. This means that $\grad F(x)=\int \grad f(x,s)\mathbf{1}_{\Delta_f}(x,s)\mu(\rd s)$ for all $x\in S_1\cap S_2$. Denote by $\nu_0$ the law of $x_0$, and since $\nu_0\ll\lambda$ by assumption, we have $\mathbb{P}^\nu(x_0\in S_1\cap S_2)=1$, and thus
\begin{equation}
    x_1=x_1\mathbf{1}_{S_1\cap S_2}(x_0)=\exp_{x_0}[-\alpha\grad f(x_0,\xi_1)]\mathbf{1}_{S_1\cap S_2}(x_0)=\exp_{x_0}[-\alpha\grad f(x_0,\xi_1)],
\end{equation}
with probability one. Then $x_1$ is integrable whenever $x_0$ is and we have that $\mathbb{E}[x_1]=\exp_{x_0}[-\alpha\grad F(x_0)]$. 

Now, since we assume that $\nu_1\ll\lambda$, we can iterate the above and prove the theorem for any $x_k$. \end{proof}

\subsection{Proof of Theorem \ref{th:weakconv}}

To prove Thm. \ref{th:weakconv}, we follow the approach of \cite{bianchi2022convergence}.

The result is the application of the following two results together with the $\nu$-a.e. differentiability of $f$ at each $x_k$. In~\cite{bianchi2022convergence} it is shown that with the kernel modified for the aforementioned null sets, i.e., 
\[
K'_{\alpha}(x,g) := \mathbf{1}_{\mathcal{D}} (x)K_{\alpha}(x,g)+\mathbf{1}_{\mathcal{D}^c}(x) g(x-\alpha \phi(x))
\]
with $\mathcal{D}\subset \mathcal{M}$ with $F(x)$ differentiable and $\phi(x)\in F(x)$ a measurable selection,
we satisfy the conditions of~\cite[Assumption RM]{bianchi2019constant}, stated as,
\begin{proposition} 
For $x\in\CM$, the maps,
\[
h_{\alpha}(x,s) = -\mathbf{1}_{\mathcal{D}} \grad F(x)-\mathbf{1}_{\mathcal{D}^c} \phi(x),\,h(x,s)\in H(x,s)=H(x):=-\partial F(x)
\]
satisfy 
\begin{enumerate}
    \item For all measureable $g:\mathcal{M}\to \mathbb{R}$,
    \[ 
    \frac{K_{\alpha} g(x)}{\alpha}:=\int \frac{g(\exp_x\left[-\alpha\phi(x,s)\right])}{\alpha}\mu(\rd s) 
    = \int g(\exp_x\left[-h_{\alpha}(x,s)\right])\mu(\rd s)
\]
\item For every $s$ $\mu$-a.e. and for all $(u_n,\alpha_n)\to(u^*,0)$ on $\mathcal{M}\times(0,\alpha_0)$,
\[
\lim\limits_{n\to\infty} \Gamma_{u_n}^{u^*}h_{\alpha_n}(u_n,s) = H(u^*,s)
\]
\item 
For $s$ a.e., $H(\cdot,s)$ is proper, upper semicontinuous, and with closed convex values.
\item For all $x\in\mathcal{M}$, $H(x,\cdot)$ is $\mu$-integrable. 
\item For every $T>0$ and every compact set $\CK\subset \mathcal{M}$,
\[
\sup \{ \bd(x(0),x(t)):t\in [0,T], x\in \mathcal{S}_H(a),a\in \CK\}< \infty.
\]
\item For every compact set $\CK\subset \mathcal{M}$, there exists $\epsilon_\CK>0$ such that, for any metric connection $\Gamma$,
\[
\sup\limits_{x\in \CK} \sup\limits_{0<\alpha<\alpha_0} \mathbb{E}_{y\in\mathcal{M}} \left[ \left(\frac{\bd(y,x)}{\alpha}\right)^{1+\epsilon_\CK} \vert x \right]<\infty
\] 
\[
\sup\limits_{x\in \CK}\sup\limits_{0<\alpha<\alpha_0} \int \left\|\Gamma_x^{x_0} h(x,s)\right\|_{\mathcal{T}\mathcal{M}^*_{x_0}}\mu(\rd s)<\infty.
\]
\end{enumerate}
\end{proposition}

\begin{proof}
1. follows from the definitions of $h_{\alpha}$, while 3. follows from the general properties of the Clarke subgradient. 3. also implies 2. 4. follows from Assumption \ref{assum_f}.2. Finally, 5. and 6. are results of Assumptions \ref{assum_f}.3 and \ref{assum_f}.4. 
\end{proof}

For the next two Lemmas, we will define a new iterative sequence in $\Omega$, as well as two paths on $\mathcal{M}$.

For all $n\in\mathbb{N}$, denote by $\mathcal{F}_n\subset\mathcal{F}$ the $\sigma$-field generated by the random variables $\{x_k,0\le k\le n\}$. Let $Z_{n+1}(t)$ denote the minimal geodesic from $x_n=Z_{n+1}(0)$ to $x_{n+1}=Z_{n+1}(1)$ and $\bar{Z}_{n+1}(t)$ the path found by concatenating all such $Z_k$ from $k=1$ to $k=n+1$. 

Next, define the sequence $y_n$ by $y_0=x_0$ and $y_{n+1}\in \exp_{y_n}\left[-\alpha_n \partial F(x_n)\right]$. The corresponding sequence of geodesic paths $Y_n(t)$ and concatenations $\bar Y_n(t)$ are defined analogously with respect to the sequence $\{y_n\}$.
Finally define the operator, for given $s>0$:

We make use of the following Lemma, a minor adaptation of~\cite[Lemma 6.2]{bianchi2019constant},
\begin{lemma}\label{lemma:uniform_int}

Let $\CK\subset \mathcal{M}$ be compact and $\{\bar{\mathbb{P}}^{a,\alpha},\,a\in \CK,0<\alpha<\alpha_0\}$ be a family of probability measures on $(\Omega,\mathcal{F})$ satisfying the following integrability condition:    
\[
\sup\limits_{n\in\mathbb{N},a\in \CK,\alpha\in(0,\alpha_0)}\bar{\mathbb{E}}^{a,\alpha}\left[ \frac{1}{\alpha}L(Z_n)\mathbf{1}_{\frac{1}{\alpha}L(Z_n)>A}\right] \to 0
\]
as $A\to \infty$. Then, it holds that 
\[
\{\bar{\mathbb{P}}^{a,\alpha} (x^{-1})_{\alpha}: a\in \CK,0<\alpha<\alpha_0\}
\]
is tight. We also have
\begin{equation}
    \sup_{a\in\CK}\bar{\mathbb{P}}^{a,\alpha}\left(\max_{0\leq n\leq \floor{\tfrac{T}{\alpha}}} \int_0^1 \rd t \bd\left(\bar{Z}_{n+1}(t),\bar{Y}_{n+1}(t))|\mathcal{F}_n\right)>\epsilon\right)\to  0,
\end{equation}
as $\alpha\to0$ and for any $T>0$.
\end{lemma}

\begin{proof}
The idea of the proof is the same as for~\cite[Lemma 6.2]{bianchi2019constant} with a few technical modifications to account for the change from the Euclidean to the Riemannian case. For the  first part, we start by introducing $T>0$, $0<\delta\leq T$ and $0\leq s\leq t\leq T$, s.t. $t-s\leq \delta$. Furthermore, let $\alpha\in(0,\alpha_0)$ and $n\coloneqq \floor{\frac{t}{\alpha}}$, $m\coloneqq \floor{\frac{s}{\alpha}}$. Then, for any $R>0$
\begin{equation}
    \bd\left(x^\alpha(t),x^\alpha(s)\right)\leq \sum_{m+1}^{n+1}L(Z_k)\leq \alpha(n-m+1)R+\sum_{k=m+1}^{n+1}L(Z_k)\mathbf{1}_{\frac{1}{\alpha}L(Z_k)>R}.
\end{equation}
This implies that
\begin{equation}
    \begin{aligned}
        \bar{\mathbb{P}}^{a,\alpha}(x^\alpha)^{-1}(\{x:w_x^T(\delta)>\epsilon\})&\leq \bar{\mathbb{P}}^{a,\alpha}\left( \sum_{k=1}^{\floor{\frac{T}{\alpha}}+1}L(Z_k)\mathbf{1}_{\frac{1}{\alpha}L(Z_k)>R}>\epsilon-\delta R\right)\\
        &\leq \frac{T}{\epsilon-\delta R}\sup_{k\in\BN}\bar{\mathbb{E}}^{a,\alpha}\left(\tfrac{1}{\alpha}L(Z_k)\mathbf{1}_{\frac{1}{\alpha}L(Z_k)>R}\right),
    \end{aligned}
\end{equation}
where we used that $n-m+1\leq \tfrac{\delta}{\alpha}$ together with Markov's inequality and the assumption that $R\delta<\epsilon$. We also denoted by $w_x^T$ the modulus of continuity of any $x\in C(\BR_+,\CM)$ on $[0,T]$,
\begin{equation}
    w_x^T(\delta)\coloneqq \sup\{\bd(x(t),x(s)):|t-s|\leq \delta,\, (t,s)\in[0,T]^2\}.
\end{equation}

Picking $R=\epsilon/(2\delta)$ and using the uniform integrability we have
\begin{equation}
    \sup_{a\in\mathcal{K},\alpha\in(0,\alpha_0)}\bar{\mathbb{P}}^{a,\alpha}(x^{\alpha})^{-1}(\{x:w_x^T(\delta)>\epsilon\}\to 0,
\end{equation}
as $\delta\to 0$. Now, observe that by definition for the initial point that $\bar{K}^{a,\alpha}(x^{-1})_{\alpha}p_0^{-1}$ is tight. Noting that $\mathcal{M}$ is separable (has a second countable cover), the same argument as in~\cite[Theorem 7.3]{billingsley2013convergence} gives the statement.


For the second point we start by defining, for any $R>0$,
\begin{equation}
    M^{a,\alpha}_{n+1}\coloneqq \int_0^1\rd t \bd\left(\bar{Z}_{n+1}(t),\bar Y_{n+1}(t)\right)|\mathcal{F}_n dt
\end{equation}
\begin{equation}
    \eta_{n+1}^{a,\alpha,\leq}\coloneqq \int_0^1\rd t \bd\left(Z_{n+1}(t)\mathbf{1}_{L(Z_{n+1})\leq R},(Z_{n+1}(t)\mathbf{1}_{L(Z_{n+1})\leq R})|\mathcal{F}_n\right)
\end{equation}
and similarly for $\eta_{n+1}^{a,\alpha,>}$. 
Finally we also introduce $\bar\eta_{n+1}^{a,\alpha,\leq}$ and $\bar\eta_{n+1}^{a,\alpha,>}$ as the sum of $\eta_{k}^{a,\alpha,\leq(>)}$ from $k=1$ to $k=n+1$. Noting that, in contrast to the original~\cite{bianchi2019constant}, we have that $\eta_{k}^{a,\alpha,\leq(>)}\in\mathbb{R}^+$, we then have
\begin{equation}
    M^{a,\alpha}_{n+1}\leq \bar\eta_{n+1}^{a,\alpha,\leq}+\bar\eta_{n+1}^{a,\alpha,>}. 
\end{equation}
Writing $q_{\alpha}\coloneqq \floor{\tfrac{T}{\alpha}}+1$, Doob's martingale inequality implies the following:
\begin{equation}
    \bar{\mathbb{P}}^{a,\alpha}\left(\max_{1\leq n\leq q_{\alpha}}\alpha\bar\eta^{a,\alpha,\leq}_{n}>\epsilon\right)\leq \frac{\bar{\mathbb{E}}^{a,\alpha}\left(\alpha\bar\eta_{q_\alpha}^{a,\alpha,\leq}\right)}{\epsilon}\leq\frac{2\alpha R\sqrt{q_\alpha}}{\epsilon}\to 0,
\end{equation}
as $\alpha\to 0$. We also have, 
\begin{equation}
    \bar{\mathbb{P}}^{a,\alpha}\left(\max_{1\leq n\leq q_\alpha}\alpha\bar\eta_{n}^{a,\alpha,>}>\epsilon\right)\leq\frac{2\alpha q_\alpha}{\epsilon}\sup_{k\in\mathbb{N}}\bar{\mathbb{E}}^{a,\alpha}\left(L(Z_k)\mathbf{1}_{L(Z_k)>R}\right).
\end{equation}
As in the Euclidean case, we now only need to pick $R$ large enough, and an arbitrarily small $\delta>0$, such that the supremum on the right hand side is smaller than or equal to $\frac{\epsilon\delta}{2T+2\alpha_0}$. This tells us that the whole right hand side is smaller than or equal to $\delta$, and gives the proof. 
\end{proof}


Next we reprove, in our setting~\cite[Lemma 6.3]{bianchi2019constant}
\begin{lemma}
For any $R>0$, define $h_{\alpha,R}(a,s)\coloneqq h_\alpha(a,s)\mathbf{1}_{\bd(a,0_{\CM})\leq R}$ and let $H_{R}(x,s):=H(x,s)$ if $\bd(x,0_{\CM})<R$ and $\{0\}_{T_x\mathcal{M}}$ if $\bd(x,0_{\CM})\ge R$ and let $B_R(x):n\mapsto \begin{cases} x_n & n< \tau_R(x)\\ x_{\tau_R(x)} & \text{otherwise} \end{cases},$ with $\tau_R(x)\coloneqq \inf\{n\in\BN : \norm{x_n}>R\}$ for all $x\in\Omega$. It holds that
    for every compact set $\CK\subset \mathcal{M}$, the family, $\{\mathbb{P}^{a,\alpha}B^{-1}_R(x^{\alpha})^{-1},\,\alpha\in(0,\alpha_0),\, a\in \CK\}$ is tight, and for every $\epsilon>0$
    \[
\sup\limits_{a\in \CK} \mathbb{P}^{a,\alpha} B^{-1}_R [\bd(x^{\alpha},\mathcal{S}_{H_R}(K))>\epsilon] \to 0
    \]
    as $\alpha\to 0$.
\end{lemma}

\begin{proof}


    Let us introduce a few definitions. Recall the paths $Z_n(t)(x)$, $\bar{Z}_n(t)(x)$, $Y_n(t)(x)$ and $\bar{Y}_n(t)(x)$ defined above, where $x$ is a sequence of iterations arising in $\Omega$. Now, with $y^R_0=x_0$, define the iterations \begin{equation}
     y^R_{n+1} = \exp_{y^R_n}\left(-\alpha \int h_{\alpha,R}(x_n,s)\mu(\rd s)\right) 
    \end{equation}
    and again $Y^R_n$ and $\bar{Y}^R_n$ the geodesic path between $y^R_n$ and $y^R_{n+1}$ and their contatentation from $0$ to $n+1$, respectively.

    Letting $\bar{\mathbb{P}}^{a,\alpha}:=\mathbb{P}^{a,\alpha} B_R^{-1}$ and noting that,
    \[
    \bar{\mathbb{E}}^{a,\alpha}\left(\left(\alpha^{-1}\bd(X_{n+1},X_n)\right)^{1+\epsilon_K}\right) = \mathbb{E}^{a,\alpha}\left(\left(\alpha^{-1}\bd(X_{n+1},X_n)\right)^{1+\epsilon_K} \mathbf{1}_{\tau_R(X)>n}\right)
    \]
    and we can apply Assumption~\ref{assump:prop6}, part 2 together with Lemma~\ref{lemma:uniform_int} to deduce the fact that, for all $\epsilon>0$ and $T>0$,
    \[
    \sup\limits_{a\in K}\bar{\mathbb{P}}^{a,\alpha}\left(\max\limits_{0\le n\le \lfloor \frac{T}{\alpha}\rfloor} \alpha \bd(\bar{Z}_{n+1},\bar{Y}^R_{n+1})>\epsilon\right) \to 0
    \]
    as $\alpha\to 0$. The same Lemma also implies that $\{\bar{\mathbb{P}}^{a,\alpha} (x^{-1})_{
    \alpha}:a\in\mathcal{K},0<\alpha<\alpha_0\}$ is tight. 

    Consider any subsequence $(a_n,\alpha_n)$ with $\alpha_n\to 0$ and $a_n\in \mathcal{K}$. By the Prokhorov Theorem there exists a convergent subsequence $(a_{n_j},\alpha_{n_j})
    \to (a^*,0)$ and $\{\bar{\mathbb{P}}^{a_{n_j},\alpha_{n_j}} (x^{-1})_{\alpha_{n_j}}\to v$. By the Skorohod representation Theorem we proceed to define distributions $\hat{x}_{n_j}\to z$ with the $\hat{x}_{n_j}\sim \bar{\mathbb{P}}^{a_{n_j},\alpha_{n_j}}  (x^{-1})_{\alpha_{n_j}}$ and $z\sim v$, with $\bd(\hat{X}_{n_j},z)\to 0$. Continuing, we can assert the existence of a subsequence such that in probability, and from this another subsequence which is in almost sure, the asymptotic behavior $\bd(\bar{Z}_{n+1},\bar{Y}^R_{n+1})\to 0$ is observed. Continuing by the same line of argumentation as~\cite{bianchi2019constant}, using the Banach Alaoglu Theorem there is $v_n(s,t)\to H_{\alpha,R}(s,z)$ and $\hat{X}_{n}\to \Phi_t(x)[H_{\alpha,R}(s,z)]$, that is, the flow of $H_{\alpha,R}(s,z)$. Since $\hat{X}_n$ is convergent, we have established a subsequence satisfying the $\sup$ of the condition:
        \[
\sup\limits_{a\in \CK} \mathbb{P}^{a,\alpha} B^{-1}_R [\bd(x^{\alpha},\mathcal{S}_{H_R}(K))>\epsilon] \to 0
    \]
    for all $\epsilon>0$ and the result has been proven.

\end{proof}

With these lemmas in place, the proof of Theorem \ref{th:weakconv} follow, with minimal alterations due to our Riemannian setting, from the arguments in the proof of Thm. 5.1 in \cite{bianchi2019constant}. 



\subsection{Proof of Proposition \ref{prop:Assumption4}}
The proofs again follow the same line of reasoning as in the proofs of Props. 5 and 6 in \cite{bianchi2022convergence}, but adjusted to the Riemannian setting. 

For 1., we start by introducing the perturbed SGD sequence
\begin{equation}
    x_{n+1}=\exp_{x_n}\left[-\alpha(\phi(x_n,\xi_{n+1})+\epsilon_{n+1})\right],
\end{equation}
where $(\epsilon_n)$, $\epsilon_{n+1}\in T_{x_n}\CM$, is a sequence of centered i.i.d. random variables, independent from $\{x_0,(\xi_n)\}$, of law $\mu^d$, and such that its distribution has a continuous and positive density on $T\CM$. Denote the Markov kernel corresponding to this perturbed SGD sequence by $\tilde K$. We will show that for each $R>0$ there exists an $\varepsilon >0$ such that 
\begin{equation}
    \forall x\in\text{cl}(B(0_\CM,R)),\,\forall A\in\mathcal{B}(\CM),\,\tilde K(x,A)\geq \varepsilon\lambda(A\cap \text{cl}(B(0_\CM,1))).
\end{equation}
This then shows that Assumption \ref{assum_f} 1. holds. To this end, denote by $\rho$ the probability distribution of $\alpha\epsilon_1$. By assumption, this has a continuous density, denoted $f$, which is positive at each point of $T_{x_0}\CM$. Denote by $\theta_x$ the probability distribution of the random variable $Z=\exp_x(-\alpha\phi(x,\xi_1))$.

For a fixed $J>0$, we have by Assumption \ref{assum_f} together with Markov's inequality, that there exists a constant $C>0$ such that
\begin{equation}
    \theta_x[Z\notin \text{cl}(B(0_\CM,J))]\leq \frac{C}{J}(1+\bd(x,0_\CM)).
\end{equation}
Then, it is just a matter of choosing $J$ large enough such that we have
\begin{equation}
    \forall x\in\text{cl}(B(0_\CM,R)),\, \theta_x[Z\notin \text{cl}(B(0_{\CM},J))]<\frac{1}{2}.
\end{equation}
Due to the continuity and positivity of $f$, we can furthermore always pick $\varepsilon>0$ such that $f(u)\geq 2\varepsilon$, for $u\in\text{cl}(B(0_\CM,J+1))$. We thus have
\begin{equation}
    \begin{aligned}
        (\theta_x\otimes\rho)[\exp_x(-\alpha(\phi(x,\xi_1)+\epsilon_{1})\in A]=&\int_A \rd u\int_{\CM}\theta_x(\rd v)\int_0^1\rd tf(\gamma_{uv}(t))\\
        \geq&\int_{A\cap \text{cl}(B(0_\CM,1))}\rd u\int_{\text{cl}(B(0_\CM,J))}\theta_x(\rd v)\int_0^1\rd tf(\gamma_{uv}(t))\\
        \geq& 2\varepsilon\int_{A\cap \text{cl}(B(0_\CM,1))}\rd u\int_{\text{cl}(B(0_\CM,J))}\theta_x(\rd v)\\
        \geq&\varepsilon\lambda(A\cap \text{cl}(B(0_\CM,1))),
    \end{aligned}
\end{equation}
where $\gamma_u(t)$ is the minimal length geodesic from $\gamma_{uv}(0)=v$ to $\gamma_{uv}(1)=u$.

This gives the proof.

For 2., and 3., we note first that Lebourg's mean value theorem tells us that for each $n\in\BN$, there exists a $c_n\in [0,1]$ and a $\zeta_n\in\partial F(u_n)$, with $u_n=\exp_{x_n}[-c_n\alpha \grad f(x_n,\xi_{n+1})\mathbf{1}_{\Delta_f}(x_n,\xi_{n+1})]$, such that
\begin{equation}
    F(x_{n+1})=F(x_n)-\alpha\langle \zeta_n,P_{x_nu_n}\grad f(x_n,\xi_{n+1})\rangle\mathbf{1}_{\Delta_f}(x_n,\xi_{n+1}).
\end{equation}

The goal is to show that under the assumptions \ref{assump:prop6} we have
\begin{equation}
    K_\alpha F(x)\leq F(x)-\alpha(1-\alpha C)\mathbf{1}_{\bd(0_\CM,x)>2R}\norm{\grad F(x)}^2+\alpha^2C\mathbf{1}_{\bd(0_\CM,x)>2R}+\alpha C\mathbf{1}_{\bd(0_\CM,x)\leq 2R},
\end{equation}
for some constant $C>0$, from which propositions 2. and 3. follow easily. To this end, we show that
\begin{equation}\label{proof_assum_expect}
    \begin{aligned}
        &\mathbb{E}_n(F(x_{n+1}))\leq F(x_n)-\alpha\mathbf{1}_{\bd(x_n,0_\CM)>2R}\norm{\grad F(x_n)}^2+\alpha C\mathbf{1}_{\bd(x_n,0_\CM)\leq2R}\\
        &+\alpha^2C\mathbf{1}_{\bd(x,0_\CM)>2R}\left((1+\norm{\grad F(x_n)})\left(\int\norm{\grad f(x_n,s)}\mu(\rd s)\right)^{1/2}+\int\norm{\grad f(x,s)}^2\mu(\rd s)\right),
    \end{aligned}
\end{equation}
which gives the statement. To do that, we start by writing
\begin{equation}\label{splitF_proof_assum}
    \begin{aligned}
        F(x_{n+1})=&F(x_n)-\alpha \mathbf{1}_{\bd(x,0_\CM)\leq 2R}\mathbf{1}_{\bd(u_n,0_\CM)\leq R}\langle\zeta_n,P_{x_nu_n}\grad f_{n+1}\rangle\mathbf{1}_{\Delta_f}(x_n,\xi_{n+1})\\
        &-\alpha\mathbf{1}_{\bd(x,0_\CM)\leq 2R}\mathbf{1}_{\bd(u_n,0_\CM)> R}\langle\zeta_n,P_{x_nu_n}\grad f_{n+1}\rangle\mathbf{1}_{\Delta_f}(x_n,\xi_{n+1})\\
        &-\alpha \mathbf{1}_{\bd(x,0_\CM)> 2R}\mathbf{1}_{\bd(u_n,0_\CM)\leq R}(\langle\zeta_n,P_{x_nu_n}\grad f_{n+1}\rangle-\langle \rd F(x_n),\grad f_{n+1}\rangle) \\
        &-\alpha \mathbf{1}_{\bd(x,0_\CM)> 2R}\mathbf{1}_{\bd(u_n,0_\CM)> R}(\langle \rd F(u_n),P_{x_nu_n}\grad f_{n+1}\rangle-\langle\rd F(x_n),\grad f_{n+1}\rangle)\\
        &-\alpha \mathbf{1}_{\bd(x,0_\CM)> 2R}\langle \rd F(x_n),\grad f_{n+1}\rangle,
    \end{aligned}
\end{equation}
where we used the notation $\grad f_{n+1}\coloneqq \grad f(x_n,\xi_{n+1})$. We will now bound all of these terms individually. We will use the same symbol $C$ for several arbitrary, but generally different, constants that show up in different steps along the way. 

For the first line, we have from Assumption \ref{assum_f} that
\begin{equation}
    \begin{aligned}
        \mathbf{1}_{\bd(u_,0_\CM)\leq R}\norm{\zeta_n}&\leq\sup_{\bd(x,0_\CM)\leq R}\norm{\partial F(x)}\leq \sup_{\bd(x,0_\CM)\leq R}\int \norm{\partial f(x,s)}\mu(\rd s)\\
        &\leq \sup_{\bd(x,0_\CM)\leq R}\int\kappa(x,s)\mu(\rd s)\leq C,
    \end{aligned}
\end{equation}
which gives
\begin{equation}\label{proof_assum_1}
    \alpha\mathbf{1}_{\bd(x,0_\CM)\leq 2R}\mathbf{1}_{\bd(u_n,0_\CM)\leq R}|\langle\zeta_n,\grad f(x_n,\xi_{n+1})\rangle|\leq \alpha C\mathbf{1}_{\bd(x,0_\CM)\leq 2R}\norm{\grad f(x_n,\xi_{n+1})}.
\end{equation}

For the second line of \eqref{splitF_proof_assum} we can again use Assumption \ref{assum_f} to find
\begin{equation}\label{proof_assum_2}
    \begin{aligned}
        &\alpha\mathbf{1}_{\bd(x,0_\CM)\leq 2R}\mathbf{1}_{\bd(u_n,0_\CM)> R}|\langle\zeta_n,P_{x_nu_n}\grad f_{n+1}\rangle\mathbf{1}_{\Delta_f}(x_n,\xi_{n+1})|\\
        \leq&\alpha\mathbf{1}_{\bd(x,0_\CM)\leq 2R}\mathbf{1}_{\bd(u_n,0_\CM)> R}\norm{\grad F(u_n)}\norm{\grad f_{n+1}}\\
        \leq&\alpha\mathbf{1}_{\bd(x,0_\CM)\leq 2R}C(\bd(x_n,0_\CM)+\alpha\norm{\grad f_{n+1}})\norm{\grad f_{n+1}}\\
        \leq&\alpha C\mathbf{1}_{\bd(x,0_\CM)\leq 2R}(\norm{\grad f_{n+1}}+\alpha\norm{\grad f_{n+1}}^2).
    \end{aligned}
\end{equation}

For the third line of \eqref{splitF_proof_assum}, we start by fixing a $x_\star\notin \text{cl}(B(0_\CM,R))$. By Assumption \ref{assump:prop6} it then holds that, for each $x\notin\text{cl}(B(0_\CM,R))$
\begin{equation}
    \norm{\grad f(x,s)}\leq \norm{\grad f(x_\star,s)}+\beta(s)\bd(x,x_\star)\leq \beta'(s)(1+\bd(x,0_\CM)),
\end{equation}
with $\beta'(s)$ square-integrable according to Assumption \ref{assum_f}. We note that
\begin{equation}
    \int\beta'(s)^2\mu(\rd s)=\int_0^\infty \mu[\beta'(\cdot)\geq \sqrt{t}]\rd t\leq \infty,
\end{equation}
such that we have $\mu[\beta'(\cdot)\geq 1/t]=o(t^2)$ for $t\to 0$. We now get
\begin{equation}
    \begin{aligned}
        \mathbf{1}_{\bd(x_n,0_\CM)>2R}\mathbf{1}_{\bd(u_n,0_\CM)\leq R}=&\mathbf{1}_{\bd(x_n,0_\CM)>2R}\mathbf{1}_{\norm{\grad f_{n+1}}\geq \tfrac{1}{\alpha}(\bd(x,0_\CM)-R)}\leq \mathbf{1}_{\bd(x_n,0_\CM)>2R}\mathbf{1}_{\beta'(\xi_{n+1})\geq \frac{\bd(x_n,0_\CM)-R}{\alpha(1+\bd(x_n,0_\CM)})}\\
        \leq&\mathbf{1}_{\bd(x_n,0_\CM)>2R}\mathbf{1}_{\beta'(\xi_{n+1})\geq\frac{R}{\alpha (1+2R)}},
    \end{aligned}
\end{equation}
where we used the definition of $u_n$ and the triangle inequality in the second step. We thus have
\begin{equation}\label{proof_assum_3a}
    \begin{aligned}
        \alpha\mathbf{1}_{\bd(x_n,0_\CM)>2R}\mathbf{1}_{\bd(u_n,0_\CM)\leq R}|\langle\zeta_n,\grad f_{n+1}\rangle|&\leq C\alpha\mathbf{1}_{\bd(x_n,0_\CM)>2R}\mathbf{1}_{\bd(u_n,0_\CM)\leq R}\norm{\grad f_{n+1}}\\
        &\leq C\alpha \mathbf{1}_{\bd(x_n,0_\CM)>2R}\norm{\grad f_{n+1}}\mathbf{1}_{\beta'(\xi_{n+1})\geq \frac{R}{\alpha(1+2R)}}.
    \end{aligned}
\end{equation}
Similarly, we have
\begin{equation}\label{proof_assum_3b}
\begin{aligned}
    &\alpha\mathbf{1}_{\bd(x_n,0_\CM)>2R}\mathbf{1}_{\bd(u_n,0_\CM)\leq R}|\langle\rd F(x_n),\grad f_{n+1}\rangle|\\\leq& C\alpha \mathbf{1}_{\bd(x_n,0_\CM)>2R}\norm{\grad F(x_n)}\norm{\grad f_{n+1}}\mathbf{1}_{\beta'(\xi_{n+1})\geq \frac{R}{\alpha(1+2R)}}.
\end{aligned}
\end{equation}

For the fourth line of \eqref{splitF_proof_assum}, we first recall that $\grad F$ is Lipschitz outside of $\text{cl}(B(0_\CM,R))$, by assumption, and we therefore have
\begin{equation}\label{proof_assum_4}
    \begin{aligned}
        &\alpha\mathbf{1}_{\bd(x_n,0_\CM)>2R}\mathbf{1}_{\bd(u_n,0_\CM)>R}|\langle\rd F(u_n),\grad f_{n+1}\rangle-\langle\rd F(x_n),\grad f_{n+1}\rangle|\\\leq& \alpha^2 C\mathbf{1}_{\bd(x_n,0_\CM)>2R}\norm{\grad f_{n+1}}^2.
    \end{aligned}
\end{equation}

Finally, we integrate Eqs. \eqref{proof_assum_1}, \eqref{proof_assum_2}, \eqref{proof_assum_3a}, \eqref{proof_assum_3b}, \eqref{proof_assum_4} and the last line of \eqref{splitF_proof_assum} with respect to $\xi_{n+1}$ to get the desired bound on the expectations. Which gives \eqref{proof_assum_expect} and finishes the proof. 

\subsection{Proof of Theorem \ref{th:convCSS}}
The proof of this theorem follows the proof of \cite[Thm. 3]{bianchi2022convergence} together with the definition of the conservative fields from above and Thm. \ref{th:weakconv}. In particular, we can note that the extra results that are needed from ergodic theory in the proof of \cite[Thm. 3]{bianchi2022convergence}, e.g. \cite[Prop. 8 and Prop. 10]{bianchi2022convergence}, hold also in the Riemannian setting. See for example \cite[Thm. 13.0.1]{meyn2012markov} and \cite{faure2013ergodic}.

\section{Appendix: Additional Definitions and Background}\label{s:add}


\paragraph{Backpropagation}
As discussed in the main text, one of the key applications of the conservative set-valued fields is to backpropagation in deep neural networks. This was the main motivations for the work in \cite{bolte2021conservative}. In practice, when we consider neural networks we would typically allow the first layer to map the input data to ordinary Euclidean space and the following layers would simply be maps between Euclidean spaces. In this way, the analysis of \cite{bolte2021conservative} will go through without modification. It would however be interesting to study how our results would align with the formulations of \cite{hauser2017principles}, but we leave this for future work.

\paragraph{Remark on the assumptions \ref{assumptionsincl}}
As stated in the main text, Assumption 2 of \ref{assumptionsincl} is implied by the more informative, and standard, assumptions
\begin{equation}
    \begin{aligned}
        &\sup_k\mathbb{E}(\norm{U_{k+1}}_{g_{x_k}}^q)<\infty,\quad \text{for some } q\geq 2,\\
        &\sum_k\alpha_k^{1+q/2}<\infty,
    \end{aligned}
\end{equation}
where $U_{k+1}$ is the noise component of the right hand side of retraction process \eqref{eq:retsgd}. The proof follows directly from the Euclidean case in \cite{benaim2006dynamics} Proposition 4.2.

\section{Appendix: O-minimal Structures}\label{sec:ominimal}
For reference, we give the definition of o-minimal structures and list a few important examples. For more information we refer to \cite{van1998tame, Dries1996GeometricCA}.

\begin{definition}[O-minimal structure]
	    An o-minimal structure on $\BR$ is a sequence $\CS=(\CS_m)_{m\in\BN}$ such that for each $m\geq 1$:
	    \begin{enumerate}\setlength\itemsep{0.5em}
	        \item[{1)}] $\CS_m$ is a boolean algebra of subsets of $\BR^m$;
	        \item[{2)}] if $A\in\CS_m$, then $\BR\times A$ and $A\times \BR$ belongs to $\CS_{m+1}$;
	        \item[{3)}] $\{(x_1,\dots,x_m)\in \BR^m\,:\, x_1=x_m\}\in \CS_m$;
	        \item[{4)}] if $A\in\CS_{m+1}$, and $\pi:\,\BR^{m+1}\to\BR^m$ is the projection map on the first $m$ coordinates, then $\pi(A)\in\CS_m$;
                \item[{5)}] the sets in $\CS_1$ are exactly the finite unions of intervals and points. 
	    \end{enumerate}
	\end{definition} 
	
A set $A\subseteq \BR^m$ is said to be \emph{definable} in $\CS$, or $\CS$-definable, if $A$ belongs to $\CS_m$. Similarly, a map $f:\, A\to B$, with $A\subseteq \BR^m$, $B\subseteq \BR^n$, is said to be definable in $\CS$ if its graph $\Gamma(f)\subseteq \BR^{m+n}$ belongs to $\CS_{m+n}$. When we do not wish to specify any particular structure we simply say that a definable function or set is \emph{tame}. 

Some important examples of o-minimal structures are:

\begin{itemize}
    \item The collection of semi-algebraic sets forms an o-minimal structure denoted $\Ralg$.

    \item Adjoining the collection of semi-algebraic sets with the graph of the real exponential function, $x\mapsto e^x$, $x\in\BR$, gives an o-minimal structure denoted $\Rexp$.

    \item The collection of restricted analytic functions can also be adjoined with $\Ralg$ to give the o-minimal structure $\Ran$. 

    \item Finally, we can combine $\Rexp$ with $\Ran$ to get the o-minimal structure $\Ranexp$. Note however, that the fact that this is again an o-minimal structure, is a highly non-trivial statement \cite{wilkie1996model}. 
\end{itemize}

The sets of functions definable in o-minimal structures includes more or less any functions that appear in modern machine learning applications. 


\end{document}